\documentclass[sn-apa]{sn-jnl}%

\usepackage{multirow}
\usepackage{colortbl}
\usepackage{subcaption}
\jyear{2022}%

\theoremstyle{thmstyleone}%
\newtheorem{theorem}{Theorem}
\newtheorem{proposition}[theorem]{Proposition}%

\theoremstyle{thmstyletwo}%
\newtheorem{example}{Example}%

\theoremstyle{thmstylethree}%

\newcommand{\tr}{\mathcal{T}}
\newcommand{\dl}{\mathcal{D}}

\newcolumntype{M}[1]{>{\centering\arraybackslash}m{#1}}

\begin{document}

\title[Explainable Models via Compression of Tree Ensembles]{Explainable Models via Compression of Tree Ensembles}

\author*[1]{\fnm{Siwen} \sur{Yan}}\email{Siwen.Yan@utdallas.edu}

\author[1]{\fnm{Sriraam} \sur{Natarajan}}\email{Sriraam.Natarajan@utdallas.edu}

\author[2]{\fnm{Saket} \sur{Joshi}}\email{saketjoshi@gmail.com}

\author[3]{\fnm{Roni} \sur{Khardon}}\email{rkhardon@iu.edu}

\author[4]{\fnm{Prasad} \sur{Tadepalli}}\email{tadepall@engr.orst.edu}

\affil*[1]{\orgdiv{Department of Computer Science}, \orgname{University of Texas at Dallas}, \orgaddress{\street{800 W Campbell Rd}, \city{Richardson}, \postcode{75080}, \state{TX}, \country{USA}}}

\affil[2]{\orgname{Amazon}, \country{USA}}

\affil[3]{\orgdiv{Department of Computer Science}, \orgname{Indiana University}, \orgaddress{\street{Luddy Hall, 700 N.\ Woodlawn Avenue}, \city{Bloomington}, \postcode{47408}, \state{IN}, \country{USA}}}

\affil[4]{\orgdiv{Department of Computer Science}, \orgname{Oregon State University}, \orgaddress{\street{1500 SW Jefferson Way}, \city{Corvallis}, \postcode{97331}, \state{OR}, \country{USA}}}

\abstract{Ensemble models (bagging and gradient-boosting) of relational decision trees have proved to be 
one of 
the most effective learning methods in the area of probabilistic logic models (PLMs). While effective, they lose one of the most important aspect of PLMs -- interpretability. 
In this paper we consider the problem of compressing a large set of learned trees into a single explainable model. To this effect, we propose CoTE -- Compression of Tree Ensembles -- that produces a single small decision list as a compressed representation.
CoTE first converts the trees to decision lists and then performs the combination and compression with the aid of the original training set. An experimental evaluation demonstrates the effectiveness of CoTE in several benchmark relational data sets.}

\keywords{Tree Combination, Ensemble, TILDE, First Order Logic}

\maketitle
\pagestyle{plain}
\section{Introduction}

Probabilistic relational logic models (PLMs) provide distributions over logical representations and have been very successful in probabilistic modeling~\citep{srlBook,starAIBook}. In contrast to other successful representations like Neural Networks, a logical representation lends itself to easy interpretation, which is an important factor for the acceptance of ML based solutions in spaces of high accountability (e.g. critical medicine).
This advantage of logical representations breaks down under ensemble models like bagged or boosted logical regression trees, because ensembles are not logical models themselves and have more complex semantics. This is somewhat unfortunate because ensembles often improve performance over single models~\citep{natarajanMLJ12}. We investigate an approach to translate the ensemble of logical trees into an easily interpretable single relational decision list. 

Previous approaches to convert complex  models into simpler ones have largely relied on using the complex model to label its own training data and subsequently train the simple model from it~\citep{craven96}. This approach has evolved into knowledge distillation \citep{hinton2015distilling}, a popular technique in the Neural Network community. With knowledge distillation, however, the focus is on reducing model complexity, not improving interpretability. Furthermore, any model-agnostic method for model compression is unable to exploit model specific features, e.g., the logical semantics of the component trees in an ensemble. In contrast, our approach compresses an ensemble of logical trees into a logical decision list (semantically equivalent to a logical decision tree) and avoids re-training the model. 
A decision list can be considered as an ordered list of first-order logic rules. 
When evaluating a decision list, the first rule satisfied by an instance is used to generate the prediction or probability for the instance.

Another line of work \citep{vidal20} recently showed how tree ensembles can be combined into a single tree without re-training. Their approach, however, is strictly for propositional representations. In contrast, we allow and exploit rich relational structure in our model.

Our incremental compression approach called {\it Compression of Tree Ensembles} (CoTEs) is based on the following key ideas: 

\begin{enumerate}
\item We exploit the similarity of the semantics of relational  trees and decision lists which allows us to soundly translate single tree into a decision list without introducing complicating factors such as negations and existential quantification. 

\item We compose trees incrementally interspersed by subsumption inference, which simplifies the rules as much as possible between successive steps. 

\item We employ the training data to simplify the decision list by pruning rules which are not supported by the data.

\item We speed up the computation by caching results of coverage and subsumption of partial rules and reusing these to evaluate combined rules.  

\end{enumerate}

We evaluate CoTE on several benchmark relational domains on multiple target relations and show that it produces compact decision lists which are close to the original bagged or boosted decision trees in predictive performance and outperforms an approximation technique that is based on relabeling the training data and learning a single tree. 

The rest of the paper is organized as follows: we first review the background and then present the problem of combining logical decision trees. We present the over all algorithm along with two different types of compression methods. Then we evaluate the algorithm empirically before concluding by outlining the areas for future research.

\section{Background and Related Work}

Compressing ensemble models has a long history in classical machine learning. Typical approaches regenerate data and learn a simpler model~\citep{bastani17,zhou16,vandewiele17}.
Other work operates directly on the feature regions~\citep{hara18}, while
\cite{joly12l1} applies Lasso
regularization to set of indicator functions to compress ensemble models.  While functional gradient boosting~\cite{friedman01} based approaches have become popular for relational data~\citep{natarajanMLJ12,KhotMLJ14}, not much progress has been made on compressing these models to interpretable ones. The key reason is that simple propositional model compression approaches cannot be directly applied to relational domains. Since generating pseudo samples for relational data is difficult, \cite{vanassche07} learns a single first order model to approximate the first order ensemble by exploiting the class distributions predicted by the ensemble.

Some work in propositional domains are closely related to our work.
\cite{vidal20} uses dynamic programming to find minimal exact tree representation with post pruning, which takes exponential time in the worst case. Our subsumption reduction also finds an exact model. Our example-based pruning is an alternative approach whose output is bounded by the number of training instances although not guaranteed to generate the minimal tree. 
In addition, our methods preserve numerical scores when combining the trees while other approaches only aim to retain the classification label.
\cite{sirikulviriya11} performs incremental heuristic rule combination where a rule is a path from the root to a leaf. It defines the conditions of redundancy and conflicts to remove rules that are not necessary to constrain prediction over instance space, which has similarity to our subsumption reduction.
\cite{hara18} uses independent rules while we use decision lists as explained next.

\section{Compression of Logical Decision Trees}

\begin{figure}[!t]
\centering
\includegraphics[width=0.95\columnwidth]{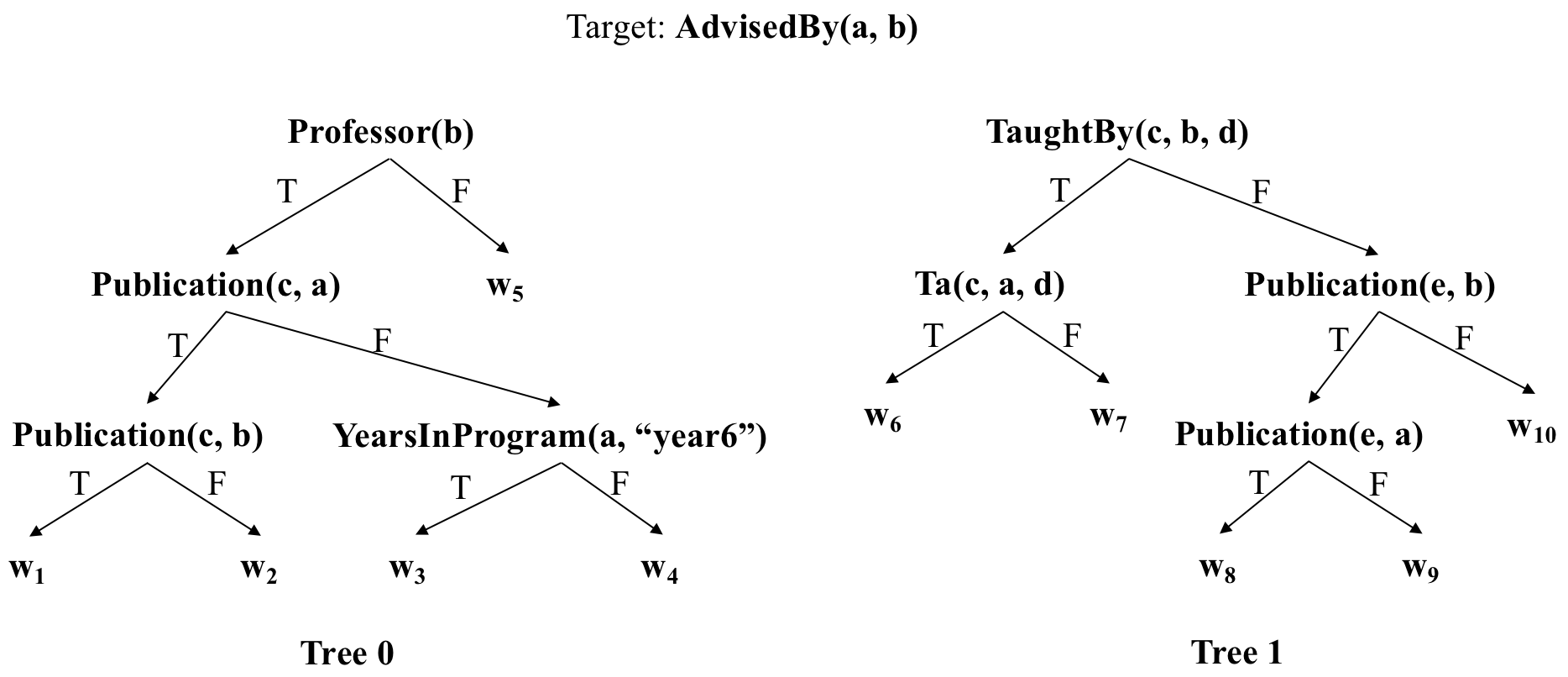}
\caption{Example of relational boosted trees. The target relation is {\bf AdvisedBy($\mathbf{a}, \mathbf{b}$)}. Each node contains a relation/predicate as condition. 
To evaluate a tree, if an 
instance satisfies the relation in the context of the path from the root, then it move to the left branch, otherwise the right branch. 
}
\label{fig:example}
\end{figure}

{\bf Notation:}
We use uppercase letters for constants and  predicate names, and lowercase letters for variables. 
A (logical) \textbf{predicate} is of the form $\mathcal{P}(t_1, \dots, t_k)$ where $\mathcal{P}$ is a predicate and $t_i$ are {\bf arguments} or logical variables. 
A \textbf{grounding} of a predicate with variables $v_1, \dots, v_k$ is a substitution $\{\langle v_1, \dots, v_k \rangle \slash \langle V_1, \dots, V_k \rangle\}$
mapping each of its variables to a constant in the domain of that variable.

\subsection{Problem Setup and Preliminaries} \label{problem-setup}

Given a training set, we use gradient boosting or bagging to generate a set of relational trees~\citep{Tilde}. These TILDE trees include logical variables which are implicitly quantified. 
\begin{example}
\label{ex:combine}
As a running example, consider the two trees presented in Figure~\ref{fig:example}. The goal is to predict if a person $\mathbf{a}$ is {\bf AdvisedBy} $\mathbf{b}$. The predicates used are {\bf Professor($\mathbf{b}$)}: $\mathbf{b}$ is a professor; {\bf Publication($\mathbf{c}, \mathbf{a}$)}: $\mathbf{c}$ is one of $\mathbf{a}$'s publication; {\bf YearsInProgram($\mathbf{a}$, ``year6'')}: $\mathbf{a}$ has been in a program for 6 years; {\bf TaughtBy($\mathbf{c}, \mathbf{b}, \mathbf{d}$)}: a course $\mathbf{c}$ is taught by $\mathbf{b}$ in a university quarter $\mathbf{d}$; {\bf Ta($\mathbf{c}, \mathbf{a}, \mathbf{d}$)}: $\mathbf{a}$ is a teaching assistant of a course $\mathbf{c}$ in a university quarter $\mathbf{d}$. For simplicity of exposition, we will ignore the weights since in the case of boosting, they are combined using the {\em sum} operation while for bagging, they are combined using {\em average}. 
\end{example}

The semantics of TILDE trees are such that at a node one evaluates the set of positive atoms on the path leading to this node as an existential formula. If this is satisfied then we take the true (left) branch and otherwise the false (right) branch. As a result the formula that describes when a leaf in the tree is reached includes the positive conjunction of all atoms exited via their left branch 
and the conjunction of negations of existential formulas, one for each false branch taken on the path to the leaf. 
For example the path formula for leaf $w_2$ of tree 0 in Figure~\ref{fig:example} is 
$[\exists c \; Professor(b) \wedge Publication(c, a)] \wedge [\neg{\exists c \; Professor(b) \wedge Publication(c, a) \wedge Publication(c, b)}]$ 
and 
the path formula for leaf $w_9$ of tree 1 is 
$[\exists e \; Publication(e, b)] 
\wedge 
[\neg{\exists c,d \; TaughtBy(c,b,d) }]
\wedge 
[\neg{\exists e \; Publication(e, b) \wedge Publication(e, a)}]$. 
The negated portions imply that the path formulas are mutually exclusive, but they have a complicated form. As we argue below, a relational decision list representation would naturally avoid the complexity of negative predicates in the path formulas, since each rule is applied only when the previous rules are inapplicable. 
We can now define our task:

\smallskip
\noindent \fbox{
\parbox{0.97\columnwidth}{
\noindent {\bf Given}: A set of $m$ learned logical decision trees $\langle \tr_0...\tr_{m-1}\rangle$. \\
\noindent{\bf Task:} Return a {\em compressed}  representation that is easily explainable and that {\em captures} the ensemble model.} 
}
\smallskip

Our compressed representation is a decision list $\dl$ where each rule in the list is a Horn clause i.e., it has only positive atoms in its body. 
Such representations have a long history in AI and cognitive science ~\citep{KlahrLaNe86,quinlan1987generating,newell90,laird12}.
We consider two alternatives for {\em captures}: in the first, the two representations must be logically equivalent. In the second, we only require the representations to have the same predictions on the original training set that produced  the ensemble.

Before presenting our algorithm consider a naive approach to the combination of propositional trees. In this case one can take a cross product of all paths in all trees to generate either one large tree (by hanging tree 1 at each leaf of tree 0, etc for the remaining trees), or a large set of rules that represent these mutually exclusive paths. This approach has two major problems. The first is that the resulting size is exponentially large. This cannot be avoided in the worst case, and in fact it is easy to show that the combination problem is co-NP hard (see \cite{vidal20}). The second problem, for the relational case, is that due to first order  quantification testing example coverage and testing of implication between rules is computationally hard. 

Our algorithm,
 \underline{C}ompression \underline{o}f (relational) \underline{T}ree \underline{E}nsembles (CoTE),
mitigates these difficulties through three high-level steps: (1) conversion of the trees to a set of decision lists with only positive atoms, (2) preprocessing the decision lists and caching subsumption and coverage results to speed up the computation, and (3) incremental combination of the trees. 
In addition, by using example based pruning we guarantee that the output size is always bounded, leading to significant compression and additional speedup.

\subsection{Conversion to Positive Decision Lists}

The first step is based on the following proposition: 
\begin{proposition}
A decision list with paths from a TILDE tree given in lexicographical ordering (where True branches are explored before False branches) and where path formulas drop all negative portions is logically equivalent to the original tree.
\end{proposition}
\begin{proof}
First translate the decision tree $\tr$ into a decision list $\dl$ with negations, where rules are ordered by their lexicographical (left child before right child) ordering. Since paths are mutually exclusive the result is equivalent to the tree. As stated in Section~\ref{problem-setup}, every path in the $\tr$ is equivalent to a formula of the form $[\exists pos(v_1)] \wedge [\wedge_i \neg \exists neg(v_1, v_{2i})]$ where $pos()$ and $neg()$ are conjunctions of atoms. Our positive $\dl$ drops the negated portions from the rules.

Consider case where $\dl$ rule is active, i.e., it is satisfied and no rule above it is satisfied. We claim that the corresponding path in the tree is satisfied. First note that the positive portions are identical. Now, by way of contradiction, suppose that one of the negative portions $\exists neg(v_1, v_{2i})$ is true. In that case one of the earlier branches in the tree, whose positive portion is exactly $neg(v_1, v_{2i})$ is satisfied.
This implies that one of the earlier rules in the positive $\dl$ is satisfied.

Next consider case where a tree path is active. In that case, the positive portion of the corresponding $\dl$ rule is satisfied, so we only need to show that no rules above it is active. To see this, note that each positive $\dl$ rule above this rule includes a superset of some $\exists neg(v_1, v_{2i})$ from this rule.
\end{proof}

\begin{example}
Returning to the example, the decision list $\dl_0$ corresponding to the first tree $\tr_0$ is as follows: 
\begin{align*}
\nonumber w_1: AdvisedBy(a, b) & :- \, Professor(b) \land Publication(c, a) \land Publication(c, b)  \\
\nonumber w_2: AdvisedBy(a, b) & :- \, Professor(b) \land Publication(c, a)  \\
\nonumber w_3: AdvisedBy(a, b) &:- \, Professor(b) \land YearsInProgram(a, ``year6'')  \\
\nonumber w_4: AdvisedBy(a, b) &:- \, Professor(b) \\
\nonumber w_5: AdvisedBy(a, b) &:-
\end{align*}
Similarly, the second tree $\tr_1$ will have five clauses in the decision list $\dl_1$ corresponding to the five paths. 
This produces identical values to the original tree on all possible examples.
\end{example}
As a result, our algorithm  uses simple subsumption between sets of positive atoms \citep{stickel92,plotkin1970note,de1997theta}
to identify rule coverage and implications between different rules.
Once the decision list $\dl_i$ is created for each tree $\tr_i$, we can try to compress the individual clauses 
using a greedy algorithm that removes one atom at a time and checks for self-subsumption.
We note that in our experiments, such a reduction is generally not necessary due to the fact that the learned trees are typically short and do not have many redundancies. After the preprocessing, subsumption is also used to compress combined decision lists.

\subsection{Caching Results for Speedup} \label{caching-speed}

As noted above, subsumption is NP-complete. Hence we need to ensure that the clauses are as short as possible and it is wise to avoid repeated subsumption/coverage tests. Our caching ideas are based on the following proposition: 
\begin{proposition}
\label{prop:groups}
If a rule has groups of predicates that do not share variables except for variables from the head (target predicate) then the evaluation of the groups can be done separately and combined.
\end{proposition}
\begin{proof}
Fix arguments at the head of the rule to be constants and divide the predicates in the body (clause) into maximal connected subclauses (cf. connected subgraphs), where two predicates are connected if they have a variable (not constant) in common. Then due to disjointness the body is satisfied if and only if all subclauses are satisfied. 
\end{proof}
We call these {\em predicate groups} and use this proposition in two ways. First, whenever possible each clause is divided into predicate groups.  Second, a combined rule from two decision lists that are standardized apart is naturally separated into groups. With this in mind, we preprocess the decision lists to generate the following data structures that cache subsumption/example coverage results:
\begin{enumerate}
 \item {\bf PG} includes all predicate (sub)groups that appear in all clauses in all decision lists. (Algorithm~\ref{algo:prep})
 \item {\bf Subsumption:} A matrix $SM$
 is created to store the subsumption results between every pair of predicate groups in $PG$. (Algorithm~\ref{algo:subsume})
 \item {\bf Example Coverage:} 
 The set of examples covered by each predicate group in $PG$ and each clause is stored in $ES$. (Algorithm~\ref{algo:prune})
\end{enumerate}
The propositions above imply that, for SCoTE in Section~\ref{sec-scote}, subsumption between rules (in this case, combined rules) can be directly resolved through the relationship of their subparts (predicate groups) which is stored in $SM$. 

Caching allows us to actually compute subsumption relationship between predicate groups or training examples covered by each predicate group {\em only once}. Both operations are very time-consuming. The cached results are used for checking redundancy during the incremental compression.

\subsection{Incremental Merging of Lists}

The naive algorithm discussed above is indeed naive. Instead of taking a product of all trees in advance and then compressing the results one can perform the operation incrementally. 
That is we start with tree 0, combine it with tree 1 and compress the results, then continue with tree 2 etc. 
Therefore, all we need is a correct and effective procedure to combine two decision lists.

One question that arises when combining decision lists is the ordering of the combined rules.
This is in contrast with combining
sets of mutually exclusive rules where the ordering does not matter. 
To resolve this we observe the following:
\begin{proposition}
The final ordering must preserve the partial order from the original lists but the ordering of ``noncomparable'' pairs of rules is not important.
\end{proposition}
\begin{proof}
Consider two positive DLs which are standardized apart, ie. their logical variables are disjoint and consider two pairs of rules $\mathcal{A}1>\mathcal{A}2$ ($\mathcal{A}1$ above $\mathcal{A}2$) and $\mathcal{B}1>\mathcal{B}2$ on these lists. If an example satisfies both $\mathcal{A}1\wedge \mathcal{B}2$ and $\mathcal{A}2\wedge \mathcal{B}1$, then due to disjointness and 
Proposition~\ref{prop:groups}
the same example also satisfies $\mathcal{A}1\wedge \mathcal{B}1$. Therefore this example will be active for $\mathcal{A}1\wedge \mathcal{B}1$ or some rule higher in the combined list and the order of $\mathcal{A}1\wedge \mathcal{B}2$ and $\mathcal{A}2\wedge \mathcal{B}1$ will not affect the prediction on this example.
\end{proof}
For example from lists  [$\mathcal{A}1, \mathcal{A}2$] and [$\mathcal{B}1, \mathcal{B}2$] the combination ($\mathcal{A}1+\mathcal{B}1$) must be first and ($\mathcal{A}2+\mathcal{B}2$) must be last but the ordering between ($\mathcal{A}1+\mathcal{B}2$) and ($\mathcal{A}2+\mathcal{B}1$) is not important. The result will be logically equivalent regardless of this ordering.
We can therefore explore this cross product in a lexicographical manner. 
As the following example shows, even a combination of two decision lists can yield a large results. The next two subsections show the result can be compressed during construction.

\begin{example}
Returning to the example, the combined decision list $\dl_{01}$ which contains 25 combinations, corresponding to the cross-product of $\dl_0$ and $\dl_1$ is as follows for boosting (mean of values for bagging):
\begin{align*}
\nonumber w_1 + w_6: AdvisedBy(a, b) & :- \, Professor(b) \land Publication(c, a) \land Publication(c, b) \\
& \:\:\:\:\:\:\: \land TaughtBy(f, b, d) \land Ta(f, a, d)  \\
\nonumber w_1 + w_7: AdvisedBy(a, b) & :- \, Professor(b) \land Publication(c, a) \land Publication(c, b) \\
& \:\:\:\:\:\:\: \land TaughtBy(f, b, d) \\
\nonumber w_1 + w_8: AdvisedBy(a, b) &:- \, Professor(b) \land Publication(c, a) \land Publication(c, b) \\
& \:\:\:\:\:\:\: \land Publication(e, b) \land Publication(e, a) \\
\nonumber w_1 + w_9: AdvisedBy(a, b) &:- \, Professor(b) \land Publication(c, a) \land Publication(c, b)  \\
& \:\:\:\:\:\:\: \land Publication(e, b) \\
\nonumber w_1 + w_{10}: AdvisedBy(a, b) &:- \, Professor(b) \land Publication(c, a) \land Publication(c, b)  \\
\nonumber w_2 + w_6: AdvisedBy(a, b) & :- \, Professor(b) \land Publication(c, a)  \\
& \:\:\:\:\:\:\: \land TaughtBy(f, b, d) \land Ta(f, a, d)  \\
\nonumber w_2 + w_7: AdvisedBy(a, b) & :- \, Professor(b) \land Publication(c, a) \land TaughtBy(f, b, d) \\
\nonumber w_2 + w_8: AdvisedBy(a, b) &:- \, Professor(b) \land Publication(c, a) \land Publication(e, b) \\
& \:\:\:\:\:\:\: \land Publication(e, a) \\
\nonumber w_2 + w_9: AdvisedBy(a, b) &:- \, Professor(b) \land Publication(c, a) \land Publication(e, b) \\
\nonumber w_2 + w_{10}: AdvisedBy(a, b) &:- \, Professor(b) \land Publication(c, a) \\
\nonumber w_3 + w_6: AdvisedBy(a, b) & :- \, Professor(b) \land YearsInProgram(a, ``year6'') \\
& \:\:\:\:\:\:\: \land TaughtBy(f, b, d) \land Ta(f, a, d) \\
\nonumber w_3 + w_7: AdvisedBy(a, b) & :- \, Professor(b) \land YearsInProgram(a, ``year6'')  \\
& \:\:\:\:\:\:\: \land TaughtBy(f, b, d) \\
\nonumber w_3 + w_8: AdvisedBy(a, b) & :- \, Professor(b) \land YearsInProgram(a, ``year6'')  \\
& \:\:\:\:\:\:\: \land Publication(e, b) \land Publication(e, a) \\
\nonumber w_3 + w_9: AdvisedBy(a, b) & :- \, Professor(b) \land YearsInProgram(a, ``year6'')  \\
& \:\:\:\:\:\:\: \land Publication(e, b) \\
\nonumber w_3 + w_{10}: AdvisedBy(a, b) & :- \, Professor(b) \land YearsInProgram(a, ``year6'') \\
\nonumber w_4 + w_6: AdvisedBy(a, b) &:- \, Professor(b) \land TaughtBy(f, b, d) \land Ta(f, a, d) \\
\nonumber w_4 + w_7: AdvisedBy(a, b) &:- \, Professor(b) \land TaughtBy(f, b, d) \\
\nonumber w_4 + w_8: AdvisedBy(a, b) &:- \, Professor(b) \land Publication(e, b) \land Publication(e, a) \\
\nonumber w_4 + w_9: AdvisedBy(a, b) &:- \, Professor(b) \land Publication(e, b) \\
\nonumber w_4 + w_{10}: AdvisedBy(a, b) &:- \, Professor(b) \\
\nonumber w_5 + w_6: AdvisedBy(a, b) &:- \, TaughtBy(f, b, d) \land Ta(f, a, d) \\
\nonumber w_5 + w_7: AdvisedBy(a, b) &:- \, TaughtBy(f, b, d) \\
\nonumber w_5 + w_8: AdvisedBy(a, b) &:- \, Publication(e, b) \land Publication(e, a) \\
\nonumber w_5 + w_9: AdvisedBy(a, b) &:- \, Publication(e, b) \\
\nonumber w_5 + w_{10}: AdvisedBy(a, b) &:- 
\end{align*}
\end{example}

\subsection{Subsumption based Compression} \label{sec-scote}

For the algorithm in this section, our goal is to compress the decision list but maintain logical equivalence. This relies on subsumption tests. 
In particular, compression is used in two places. The first compression is within a rule where we can remove individual predicates (during preprocessing) or subgroups of predicates during the combination of rules. For instance, in the above example, the ($w_1 + w_8$) clause 
can drop $Publication(e, a)$ and $Publication(e, b)$ and
be reduced to (similar to the ($w_1 + w_9$) clause)
\begin{align}
    w_1 + w_8: AdvisedBy(a, b) &:- \, Professor(b) \land Publication(c, a) \nonumber \\
    & \:\:\:\:\:\:\: \land Publication(c, b).
\end{align}
The second compression is removal of rules which are subsumed by rules that appear above them in the list. 
As mentioned above all these decisions can be done directly using $SM$ without rerunning subsumption tests.
For instance, the clause in Equation 1 subsumes the ($w_1 + w_{10}$) clause (as well as the ($w_1 + w_9$) clause) from Example 3. So now the subsumption reduction is applied and the decision list will be simplified to (we show only rules with $w_1$ for brevity)
\begin{flalign*}
\nonumber w_1 + w_6: AdvisedBy(a, b) & :- \, Professor(b) \land Publication(c, a)  \\
& \:\:\:\:\:\:\: \land Publication(c, b) \land TaughtBy(f, b, d) \\
& \:\:\:\:\:\:\: \land Ta(f, a, d)  \\
\nonumber w_1 + w_7: AdvisedBy(a, b) & :- \, Professor(b) \land Publication(c, a)  \\
& \:\:\:\:\:\:\: \land Publication(c, b) \land TaughtBy(f, b, d) \\
\nonumber w_1 + w_8: AdvisedBy(a, b) &:- \, Professor(b) \land Publication(c, a)  \\
& \:\:\:\:\:\:\: \land Publication(c, b) \\
& ......
\end{flalign*}

We refer to this algorithm that uses subsumption for reduction as {\sc SCoTE}.
As shown in our experiments this algorithm produces a logically equivalent expression that provides a significant compression relative to the naive algorithm. 

\subsection{Example based Compression} \label{sec-ecote}

While subsumption often produces small lists this is not always the case, and we know that in the worst case with arbitrary trees the results can be of exponential size. However, recall that we are generating a decision list 
from the original given dataset. If the dataset has $N$ examples, then the final representation should not have more than $N$ distinctions, $N$ rules in cases of a list and $N$ leaves if we were using a tree. If the intermediate representation has more rules, then the cases they cover have not been validated by the data. They either correspond to situations that do not naturally arise in the data, or their prediction has not been validated and is hence not trustworthy. Although we are preforming a logical transformation of the ensemble,
our idea is to: {\em prune rules in the decision list which do not cover any example in the original dataset}. 
If the original ensemble has good predictions they are likely from the retained rules so we expect to retain predictive accuracy while avoiding excessive size. 

The incremental algorithm works in exactly the same manner as in {\sc SCoTE}. But the decisions on compression are based on the coverage of examples instead of subsumption between rules.
To distinguish from {\sc SCoTE}, we refer to this algorithm as {\sc ECoTE}.
While  {\sc SCoTE} preserves {logical equivalence} with the learned ensemble, 
{\sc ECoTE} 
guarantees that training {examples coverage} is the same as the original model. 

{\sc ECoTE} can reduce rules further than {\sc SCoTE}. For example, in our training data, any person who is a professor has some publication. Then removing $Publication(e, b)$ does not affect example coverage of the ($w_2 + w_9$) clause. Hence we get
\begin{flalign*}
\nonumber w_2 + w_9: AdvisedBy(a, b) & :- \, Professor(b) \land Publication(c, a)
\end{flalign*}

{\sc ECoTE} also detects combined rules with empty coverage relative to the dataset. Such rules are needed to maintain logical equivalence but not relative to our dataset.
For example, professors publish papers with students in departments and teach courses. The students publish papers with a professor always have been a teaching assistant in at least one of the professor's course. Under this situation, the ($w_1 + w_6$) clause covers all positive examples that the ($w_1 + w_7$) clause can cover. If an example fails the ($w_1 + w_6$) clause, it will also fails the ($w_1 + w_7$) clause. Hence the example coverage of the clause ($w_1 + w_7$) is empty and the clause can be removed.

\subsection{Algorithm}
\begin{algorithm}[!t]
\caption{CoTE: \underline{\textbf{C}}\underline{\textbf{o}}mpression of \underline{\textbf{T}}ree \underline{\textbf{E}}nsembles}
\label{algo:main}
\small
\begin{algorithmic}[1]
\Procedure{CoTE}{Ensemble model $\mathcal{M}$}
\State $\mathcal{L}$, $PG$ $\leftarrow$ PREP($\mathcal{M}$) \Comment{See Alg \ref{algo:prep}}
\If{logic-equivalent}
\State $L$ $\leftarrow$ SR($\mathcal{L}$, $PG$) \Comment{See Alg \ref{algo:subsume}}
\ElsIf{data-equivalent}
\State $L$ $\leftarrow$ EP($\mathcal{L}$, $PG$) \Comment{See Alg \ref{algo:prune}}
\EndIf
\State \textbf{return} {$L$}
\EndProcedure
\end{algorithmic}
\end{algorithm}

Combining all the observations above we get our promised compression algorithms. 
Algorithm~\ref{algo:main} presents the {\em Compression of Tree Ensembles} (CoTE) algorithm. 
Preprocessing into decision lists and predicates groups is performed in Algorithm~\ref{algo:prep}.
\begin{algorithm}[!t]
\caption{SR: \underline{\textbf{S}}ubsumption \underline{\textbf{R}}eduction}
\label{algo:subsume}
\small
\begin{algorithmic}[1]
\Procedure{SR}{Decision lists $\mathcal{L}$, Predicate groups $PG$}
\State $M$ $\leftarrow$  size($PG$)
\For{$i \in 0$ \textbf{to} $M-1$}
  \For{$j \in 0$ \textbf{to} $M-1$}
    \State ${SM}_{i,j}$ $\leftarrow$ isSubsume(${PG}_{i}$, ${PG}_{j}$)
  \EndFor
\EndFor

\State $N$ $\leftarrow$  size($\mathcal{L}$)
\State $L \leftarrow {\mathcal{L}}_{0}$
\For{iter $i \in 1$ \textbf{to} $N-1$}
  \State ${L}^{*} \leftarrow \emptyset$
  \For{each ${\mathtt{C}}_{j}$ $\in$ $L$}
    \For{each ${\mathtt{C}}_{k}$ $\in$ ${\mathcal{L}}_{i}$}
      \State $\mathtt{C}$ $\leftarrow$ addClauses(${\mathtt{C}}_{j}$, ${\mathtt{C}}_{k}$)
      \State ${L}^{*}$ $\leftarrow$ append(${L}^{*}$, reduceClause($\mathtt{C}$, $SM$)) \Comment{In clause}
    \EndFor
  \EndFor
  \State $L$ $\leftarrow$ reduceList(${L}^{*}$, $SM$) \Comment{Between clauses}
\EndFor
\State \textbf{return} $L$
\EndProcedure
\end{algorithmic}
\end{algorithm}
Algorithm~\ref{algo:subsume} describes {\sc SCoTE}.
Lines 3-7 build a matrix to store the subsumption relationship between any pair of predicate groups. 
In line 15, the combined clause is shortened by checking subsumption relationship between its predicate groups before it is appended to our decision list. 
In line 18, clauses that are subsumed by any clause higher in the decision list are removed. 
\begin{algorithm}[!t]
\caption{EP: \underline{\textbf{E}}xample based \underline{\textbf{P}}runing}
\label{algo:prune}
\small
\begin{algorithmic}[1]

\Procedure{EP}{Decision lists $\mathcal{L}$, Predicate groups $PG$}
\State $ES \leftarrow \emptyset$  \Comment{Examples coverage of $PG$}
\State $N$ $\leftarrow$  size($\mathcal{L}$)
\For{$i \in 0$ \textbf{to} $N-1$}
  \State $ES$ $\leftarrow$ append($ES$, expCover($D_{train}$, ${\mathcal{L}}_{i}$, $PG$))
\EndFor

\State $L \leftarrow {\mathcal{L}}_{0}$
\For{iter $i \in 1$ \textbf{to} $N-1$}
  \State ${L}^{*} \leftarrow \emptyset$
  \State $EC \leftarrow \emptyset$  \Comment{Examples coverage}
  \For{each ${\mathtt{C}}_{j}$ $\in$ $L$}
    \For{each ${\mathtt{C}}_{k}$ $\in$ ${\mathcal{L}}_{i}$}
      \State $\mathtt{C}$ $\leftarrow$ addClauses(${\mathtt{C}}_{j}$, ${\mathtt{C}}_{k}$)
      \State $EC_{\mathtt{C}} \leftarrow$ checkCoverage($\mathtt{C}$, $ES$, $EC$) \Comment{Between clauses}
      \If{$EC_{\mathtt{C}} \neq \emptyset$} 
        \State ${L}^{*}$ $\leftarrow$ append(${L}^{*}$, pruneClause($\mathtt{C}$, $ES$, $EC_{\mathtt{C}}$, $EC$))  \Comment{In clause}
      \EndIf
    \EndFor
  \EndFor
    \State $L$ $\leftarrow {L}^{*}$
\EndFor
\State \textbf{return} $L$

\EndProcedure
\end{algorithmic}
\end{algorithm}
Algorithm~\ref{algo:prune}  describes {\sc ECoTE}.
In line 5, the algorithm generates example coverage of not only each clause in each decision list, but also each predicate group in each clause. 
The example coverage of each clause is later used in line 14 to remove combined clauses with empty intersection of example coverage while also updates the overall example coverage $EC$ and record the example coverage of the clause $\mathtt{C}$ as ${EC}_{\mathtt{C}}$. 
Similarly, example coverage of predicate groups helps check whether removing a predicate group will change the example coverage of the combined clause $\mathtt{C}$ in line 13.
$ES$, ${EC}_{\mathtt{C}}$ and $EC$ are used in line 16 to prune the clause $\mathtt{C}$ without changing the example coverage of $\mathtt{C}$ before appending this clause to ${L}^{*}$. 
Algorithm~\ref{algo:prep} shows the preprocessing procedure which is common for both subsumption reduction and example based pruning. This procedure transforms the trees to a decision list, iterates through each rule, reduces it and constructs the groups of literals. 
\begin{algorithm}
\caption{Preprocess the ensemble model} \label{algo:prep}

\begin{algorithmic}[1]

\Procedure{PREP}{Ensemble model $\mathcal{M}$}
\State $N$ $\leftarrow$  size($\mathcal{M}$)
\State $\mathcal{L} \leftarrow \emptyset$
\State $PG \leftarrow \emptyset$
\For{$i\in 0$ \textbf{to} $N-1$}
  \State $L$ $\leftarrow$ transform(${\mathcal{M}}_{i}$)
  \Comment{Tree to list}
  \State $L^* \leftarrow \emptyset$
  \For{each ${\mathtt{C}}_{j}$ in $L$}
    \State ${\mathtt{C}}_{j}$ $\leftarrow$ clauseReduction(${\mathtt{C}}_{j}$) \Comment{Reduce clause}
    \State $L^*$ $\leftarrow$ append($L^*$, ${\mathtt{C}}_{j}$)
    \State $PG$ $\leftarrow$ append($PG$, literalGroups(${\mathtt{C}}_{j}$))
  \EndFor
  \State $\mathcal{L}$ $\leftarrow$ append($\mathcal{L}$, $L^*$)
\EndFor
\State \textbf{return} $\mathcal{L}$, $PG$
\EndProcedure
\end{algorithmic}
\end{algorithm}

\section{Experimental Evaluation}

We aim to answer the following questions explicitly: {\bf Q1:} Are our compression methods effective in improving explainability by reducing the number of clauses and the average clause length? {\bf Q2:} Are our compression methods faithful to the original model? {\bf Q3:} How do the methods compare against re-labeling the data and fitting a (logical) decision tree? To answer these questions, we consider $5$ standard relational data sets. We compare our {\sc SCoTE} and {\sc ECoTE} algorithms against a re-labeling method where the training data is re-labeled by the learned ensemble and a single tree is induced from this re-labeled data. We call this as {\sc soft labels}~\cite{craven96}. This is a strong baseline and is the one most often used for relational learning~\citep{natarajanMLJ12,KhotMLJ14}. 

{\bf Datasets:}  (1) The {\bf UW-CSE} data set \cite{mln06} is a standard benchmark where the task is to predict whether a professor and student share an advisor-advisee relationship based on other relationships involving professors, students, courses, publications, classes etc., in 5 areas of Computer Science.
(2) \textbf{Cora Entity Resolution} is a dataset describing the relationship between citations. The task is to identify which citations refer to the same publication.
(3) The {\bf IMDB} data set was first created by \cite{bottomupmln07} and contains relationships among movie elements like cast, genre etc. 
(4) The {\bf WebKB} data set consists of web pages and hyperlinks from 4 CS departments \citep{craven1998learning} and the task is to predict if someone is a faculty member. (5) The {\bf ICML} data set,  extracted from the Microsoft academic graph (MAG) \citep{sinha2015overview,Dhami21Gaifman}, consists of papers from ICML 2018 and the task is to predict coauthors.

\begin{sidewaystable*}
\sidewaystablefn%
\begin{minipage}[h]{\textheight}
    \centering
    \caption{Number of clauses/rules in relational domains and average rule length when the base model is {\bf Boosting} with 20 TILDE trees. Note that while subsumption compresses the model significantly, the gains are higher when using the example based reduction. ``Max clauses'' and ``avg'' length of original model are the results when the clauses are naively combined without any compression; ``max clauses'’ is the maximal possible number of clauses of combining 20 trees which is also the worst case when no compression can be applied; ``avg'’ is the average clause length in this case; ``\#paths'’ is the total number of paths of the original 20 trees; ``depth'’ is the maximal depth over the 20 trees. AUC scores close or higher to original model are bolded.} \label{tab:relational-domains-boosting}
    \scalebox{0.62}{
    \begin{tabular}{@{\extracolsep{\fill}}cM{0.15\textwidth}ccM{0.065\textwidth}M{0.05\textwidth}M{0.05\textwidth}m{0.05\textwidth}M{0.08\textwidth}M{0.04\textwidth}M{0.05\textwidth}m{0.05\textwidth}M{0.08\textwidth}M{0.04\textwidth}M{0.05\textwidth}m{0.05\textwidth}M{0.08\textwidth}M{0.04\textwidth}M{0.05\textwidth}M{0.05\textwidth}@{\extracolsep{\fill}}}
        \toprule
         \multirow{2}{*}[-1.8ex]{Domain} & \multirow{2}{*}[-1.8ex]{Relation} & \multicolumn{6}{@{}c@{}}{Original model} & \multicolumn{4}{@{}c@{}}{SCoTE} & \multicolumn{4}{@{}c@{}}{ECoTE} & \multicolumn{4}{@{}c@{}}{Soft labels}\\
        \cmidrule(lr){3-8}\cmidrule(lr){9-12}\cmidrule(lr){13-16}\cmidrule(lr){17-20}
         & & \#paths & depth & max clauses & avg & ROC & PR & \#clause & avg & ROC & PR & \#clause & avg & ROC & PR & \#clause & avg & ROC & PR \\
        \midrule
        \multirow{4}{*}[-7.3ex]{Cora} & Same-Author(A,B) & 99 & 5 & 7.6 $\times {10}^{13}$ & 58.10 & 0.726 & 0.941 & 18 & 6.83 & \textbf{0.726} & \textbf{0.941} & 9 & 4.33 & \textbf{0.726} & \textbf{0.941} & 5 & 3.00 & 0.673 & 0.938 \\
        \cmidrule{2-20}
        & Same-Bib(A,B) & 80 & 4 & 1.1 $\times {10}^{12}$ & 36.50 & 0.921 & 0.949 & 7 & 2.86 & \textbf{0.921} & \textbf{0.949} & 7 & 2.86 & \textbf{0.921} & \textbf{0.949} & 4 & 2.00 & 0.918 & \textbf{0.948} \\
        \cmidrule{2-20}
        & Same-Title(A,B) & 80 & 6 & 1.1 $\times {10}^{12}$ & 38.75 & 0.693 & 0.672 & 30 & 6.33 & \textbf{0.693} & \textbf{0.672} & 14 & 3.86 & \textbf{0.693} & \textbf{0.672} & 4 & 1.75 & 0.684 & 0.656 \\
        \cmidrule{2-20}
        & Same-Venue(A,B) & 100 & 7 & 9.5 $\times {10}^{13}$ & 56.40 & 0.806 & 0.694 & 238 & 10.70 & \textbf{0.806} & \textbf{0.694} & 52 & 6.23 & \textbf{0.805} & 0.690 & 5 & 3.20 & 0.553 & 0.416 \\
        \midrule
        \multirow{3}{*}[-5ex]{WebKB} &  Faculty(A) & 129 & 7 & 1.2 $\times {10}^{16}$ & 42.58 & 0.678 & 0.474 & 96 & 10.54 & \textbf{0.678} & \textbf{0.474} & 29 & 4.93 & \textbf{0.678} & \textbf{0.474} & 8 & 2.38 & 0.592 & 0.231 \\
        \cmidrule{2-20}
        & Course-Prof(A,B) & 175 & 6 & 6.3 $\times {10}^{18}$ & 49.60 & 0.521 & 0.365 & 36 & 4.19 & \textbf{0.521} & \textbf{0.365} & 18 & 2.78 & \textbf{0.521} & \textbf{0.365} & 9 & 2.22 & 0.491 & 0.333 \\
        \cmidrule{2-20}
        & Course-TA(A,B) & 164 & 7 & 1.5 $\times {10}^{18}$ & 56.60 & 0.575 & 0.439 & 31 & 5.35 & \textbf{0.575} & \textbf{0.439} & 12 & 2.83 & 0.548 & 0.374 & 7 & 3.71 & 0.540 & 0.384 \\
        \midrule
        UW-CSE & Advised-By(A,B) & 157 & 4 & 5.5 $\times {10}^{17}$ & 35.52 & 0.997 & 0.994 & n/a & n/a & n/a & n/a & 103 & 2.95 & \textbf{0.996} & \textbf{0.993} & 9 & 2.22 & 0.943 & 0.927 \\
        \midrule
        \multirow{3}{*}[-5ex]{IMDB} & Worked-Under(A,B) & 81 & 3 & 1.4 $\times {10}^{12}$ & 30.10 & 1.000 & 1.000 & 6 & 2.00 & \textbf{1.000} & \textbf{1.000} & 5 & 1.60 & \textbf{1.000} & \textbf{1.000} & 4 & 1.50 & \textbf{1.000} & \textbf{1.000} \\
        \cmidrule{2-20}
        & Genre(A,B) & 100 & 1 & 9.5 $\times {10}^{13}$ & 16.00 & 0.620 & 0.452 & 512 & 4.50 & \textbf{0.620} & \textbf{0.452} & 9 & 0.89 & \textbf{0.620} & \textbf{0.452} & 5 & 0.80 & 0.419 & 0.178 \\
        \cmidrule{2-20}
        & Female-Gender(A) & 133 & 4 & 2.5 $\times {10}^{16}$ & 26.70 & 0.655 & 0.633 & 1591 & 5.32 & 0.655 & \textbf{0.633} & 13 & 1.00 & \textbf{0.657} & \textbf{0.634} & 6 & 1.67 & 0.626 & 0.542 \\
        \midrule
        ICML & Co-Author(A,B) & 259 & 4 & 1.8 $\times {10}^{22}$ & 31.03 & 0.844 & 0.275 & n/a & n/a & n/a & n/a & 3972 & 3.18 & \textbf{0.801} & \textbf{0.269} & 13 & 1.46 & 0.501 & 0.022 \\
        \botrule
        \end{tabular}
        }
    \end{minipage}
\end{sidewaystable*}

{\bf Experiment setup:} We ran Boosting and Bagging algorithms with the number of trees set as 20 to yield two sets of decision trees. We then ran compression methods SCoTE and ECoTE on both Boosting and Bagging trees. We also ran soft labels method on Boosting trees. For Boosting, the size of each tree is small as weak learner. For Bagging, large TILDE trees are needed for better performance.

{\bf Results:} Table~\ref{tab:relational-domains-boosting} presents the results of compressing the boosted model with 20 TILDE trees in all the domains and compares {\sc SCoTE} and {\sc ECoTE} with soft labels. As can be noted, a simple combination can yield an exponential number of rules with large rule lengths. In all but two of the domains, {\sc SCoTE} was able to achieve significant compression. In two domains, the runtime for {\sc SCoTE} is excessive and it did not complete the compression in 10 hours. 
We also note from the results that {\sc ECoTE} is significantly better than {\sc SCoTE} in a majority of the domains.
\begin{sidewaystable*}
\sidewaystablefn%
\begin{minipage}[h]{\textheight}
    \centering
    \caption{Number of clauses/rules in relational domains and average rule length when the base model is {\bf Bagging} with 20 large TILDE trees. The results are similar to the boosted combination where {\sc ECoTE} is more efficient than {\sc SCoTE} without significance loss in performance. Column names refer to Table~\ref{tab:relational-domains-boosting}. AUC scores close or higher to original model are bolded.} \label{tab:relational-domains-bagging}
    \scalebox{0.65}{
    \begin{tabular}{@{\extracolsep{\fill}}cM{0.15\textwidth}M{0.065\textwidth}cM{0.08\textwidth}M{0.04\textwidth}M{0.05\textwidth}M{0.05\textwidth}M{0.08\textwidth}M{0.04\textwidth}M{0.05\textwidth}M{0.05\textwidth}M{0.08\textwidth}M{0.04\textwidth}M{0.05\textwidth}M{0.05\textwidth}@{\extracolsep{\fill}}}
        \toprule
        \multirow{2}{*}[-1.8ex]{Domain} & \multirow{2}{*}[-1.8ex]{Relation} & \multicolumn{6}{@{}c@{}}{Original model} & \multicolumn{4}{@{}c@{}}{SCoTE} & \multicolumn{4}{@{}c@{}}{ECoTE}\\
        \cmidrule(lr){3-8}\cmidrule(lr){9-12}\cmidrule(lr){13-16}
         & & \#paths & depth & Max clauses & avg & ROC & PR & \#clause & avg & ROC & PR & \#clause & avg & ROC & PR \\
        \midrule
        \multirow{4}{*}[-2.3ex]{Cora} & Same-Author(A,B) & 66 & 7 & 2.3 $\times {10}^{7}$ & 29.14 & 0.676 & 0.944 & 51 & 7.24 & \textbf{0.676} & \textbf{0.944} & 23 & 4.96 & \textbf{0.676} & \textbf{0.944} \\
        \cmidrule{2-16}
        & Same-Bib(A,B) & 164 & 11 & 1.1 $\times {10}^{16}$ & 60.25 & 0.927 & 0.960 & 141 & 12.86 & \textbf{0.927} & \textbf{0.960} & 40 & 6.70 & \textbf{0.926} & \textbf{0.959} \\
        \cmidrule{2-16}
        & Same-Title(A,B) & 47 & 5 & 9.7 $\times {10}^{4}$ & 17.41 & 0.656 & 0.640 & 48 & 7.21 & \textbf{0.656} & \textbf{0.640} & 15 & 3.93 & \textbf{0.656} & \textbf{0.640} \\
        \cmidrule{2-16}
        & Same-Venue(A,B) & 95 & 9 & 1.2 $\times {10}^{9}$ & 35.29 & 0.573 & 0.433 & 492 & 13.66 & \textbf{0.573} & \textbf{0.433} & 65 & 6.22 & \textbf{0.572} & 0.431 \\
        \midrule
        \multirow{3}{*}[-2.6ex]{WebKB} & Faculty(A) & 81 & 7 & 2.8 $\times {10}^{10}$ & 23.01 & 0.696 & 0.514 & 52 & 6.52 & \textbf{0.696} & \textbf{0.514} & 29 & 5.00 & \textbf{0.696} & \textbf{0.514} \\
        \cmidrule{2-16}
        & Course-Prof(A,B) & 100 & 7 & 7.5 $\times {10}^{11}$ & 31.68 & 0.506 & 0.361 & 45 & 5.11 & \textbf{0.506} & \textbf{0.361} & 19 & 2.89 & \textbf{0.507} & \textbf{0.361} \\
        \cmidrule{2-16}
        & Course-TA(A,B) & 50 & 7 & 1.2 $\times {10}^{6}$ & 17.29 & 0.490 & 0.338 & 16 & 3.44 & \textbf{0.490} & \textbf{0.338} & 9 & 2.56 & \textbf{0.490} & \textbf{0.338} \\
        \midrule
        UW-CSE & Advised-By(A,B) & 325 & 7 & 1.5 $\times {10}^{24}$ & 48.76 & 0.987 & 0.975 & n/a & n/a & n/a & n/a & 126 & 2.42 & \textbf{0.986} & 0.956 \\
        \midrule
        \multirow{3}{*}[-3.5ex]{IMDB} & Worked-Under(A,B) & 85 & 8 & 2.4 $\times {10}^{10}$ & 31.42 & 1.000 & 1.000 & 540 & 8.79 & \textbf{1.000} & \textbf{1.000} & 42 & 3.24 & \textbf{1.000} & \textbf{1.000} \\
        \cmidrule{2-16}
        & Genre(A,B) & 65 & 5 & 1.2 $\times {10}^{8}$ & 14.68 & 0.709 & 0.483 & 6444 & 8.73 & \textbf{0.709} & \textbf{0.483} & 11 & 1.36 & \textbf{0.709} & \textbf{0.483} \\
        \cmidrule{2-16}
        & Female-Gender(A,B) & 132 & 6 & 2.9 $\times {10}^{13}$ & 35.88 & 0.701 & 0.646 & n/a & n/a & n/a & n/a & 18 & 1.72 & 0.680 & 0.635 \\
        \midrule
        ICML & Co-Author(A,B) & 196 & 4 & 3.4 $\times {10}^{16}$ & 25.71 & 0.705 & 0.155 & n/a & n/a & n/a & n/a & 1738 & 2.99 & 0.684 & 0.110 \\
        \botrule
        \end{tabular}}
    \end{minipage}
\end{sidewaystable*}
Similar results can be observed when compressing bagged models with 20 large TILDE trees in Table~\ref{tab:relational-domains-bagging} where {\sc SCoTE} achieves significant compression in a majority of domains while {\sc ECoTE} is even better. 
\begin{figure*}[h]
  \centering
  \begin{subfigure}[t]{0.49\columnwidth}
    \centering
    \includegraphics[width=1.\columnwidth]{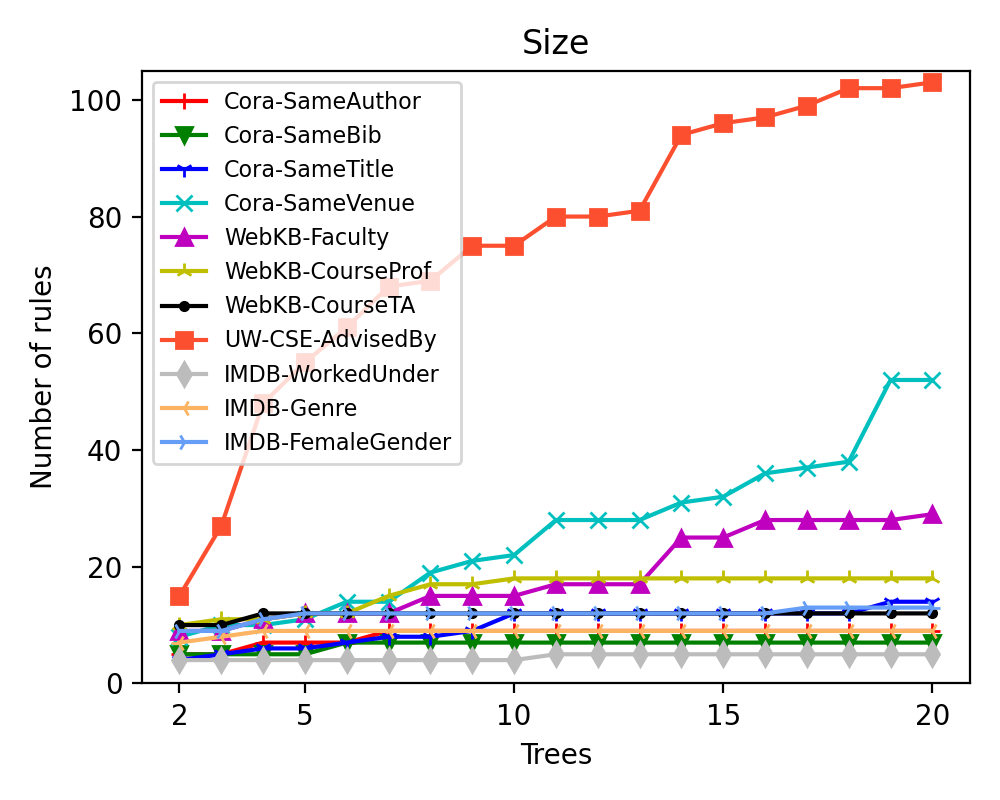}
    \label{fig:relational-domains}
  \end{subfigure}
  \hfill
  \begin{subfigure}[t]{0.49\columnwidth}
    \centering
    \includegraphics[width=1.\columnwidth]{images/cora-samevenue-aucpr.png}
    \label{fig:cora-samevenue-perform-mp}
  \end{subfigure}
  \caption{Sample results for example based compression on gradient boosted trees. (a) Number of rules vs. number of trees for 11 relational domains. (b) Cora SameVenue task. Note that the list's train and test set performance are not significantly worse than the tree performance } \label{fig:relational-boosting-results}
\end{figure*}

A sample of these results can be seen in Figure~\ref{fig:relational-boosting-results}. Figure~\ref{fig:relational-boosting-results}.$a$ presents the results of plotting the number of rules as a function of the number of trees for {\sc ECoTE}. Note that we do not plot ICML because the number of rules are quite high and it will render all the other curves flat. 
This shows that in most problem we are able to compress into a small number of rules regardless of the number of trees combined. 
The performance in a single relation of CORA is presented in Figure~\ref{fig:relational-boosting-results}.$b$. 
This illustrates that Boosting is needed
(because at least 8 trees are needed to obtain good performance), that there is no overfitting because train and test performance are close to each other, and that the decision list matches the performance of the ensemble of trees. 
More performance curves are presented in the appendix. 
In summary, {\bf Q1} can be answered affirmatively in that the algorithms achieve significant compression over simple combinations.

To understand the effectiveness, we present the test set performance of the compression algorithms in Table~\ref{tab:relational-domains-boosting} for boosted models and in Table~\ref{tab:relational-domains-bagging}. Specifically, we present the AUCPR and AUCROC obtained by the original models (boosting with 20 trees and bagging with 20 trees respectively) and compare their performance with {\sc SCoTE} and {\sc ECoTE}. It can be observed that 
as expected {\sc SCoTE} preserved the predictions of the ensemble exactly, and the score for {\sc ECoTE} are not significantly different, thus answering {\bf Q2} affirmatively.

While it appears that the soft label generation method has significantly fewer clauses with smaller lengths, it suffers from two issues -- being approximate, its performance is sometimes significantly worse than {\sc SCoTE} and {\sc ECoTE}.
In nearly all the domains, its AUCPR and AUCROC is worse than the compression methods. The second issue is that the final learned model is not representative (neither logically equivalent nor coverage equivalent) of the original model. Hence, it cannot be directly used for interpretation of the original model. Thus {\bf Q3} can be strongly answered in that the compression techniques significantly outperform the soft label strategy.

The discussion so far provided a quantitative evaluating of the compression algorithms. We next illustrate the compression qualitatively by showing a concrete outcome. 
The final combined decision list of $SameBib(a,b)$ (denoted as {\em SB}) using {\sc ECoTE} to compress $20$ boosting trees is presented below ({\em T - Title, V - Venue, A - Author, HV - HasWordVenue, HA - HasWordAuthor}). As can be seen the rules are compact and interpretable. This clearly demonstrates that it is possible to convert a potential blackbox model such as relational boosted trees and generate an interpretable and explainable model. 
\resizebox{1.\columnwidth}{!}{
\begin{minipage}{\columnwidth}
\begin{align*}
4.462: SB(a,b) & :- \, T(a,c) \land T(b,c). \\
2.684: SB(a,b) & :- \, V(b,c) \land V(a,c) \land A(b,d) \land HV(e,f) \land \\
               & \hspace{2em} A(a,d) \land HA(d,f). \\
2.667: SB(a,b) & :- \, V(b,c) \land V(a,c) \land A(b,d) \land A(a,d). \\
2.489: SB(a,b) & :- \, V(b,c) \land V(a,c). \\
0.158: SB(a,b) & :- \, A(b,c) \land HV(d,e) \land A(a,c) \land HA(c,e). \\
0.141: SB(a,b) & :- \, A(a,c) \land A(b,c). \\
0.198: SB(a,b) & :- \, . \\
\end{align*} 
\end{minipage}
}
While successful, both {\sc SCoTE} and {\sc ECoTE} can be less efficient in certain conditions. For instance, in {\em ICML}, the learned model mainly had constants in the tree nodes (as opposed to variables in other data sets). This renders the subsumption powerless as the constants cannot subsume other constants. Consequently, {\sc SCoTE} did not converge/finish even after 10 hours. While {\sc ECoTE} does not rely on subsumption and was able to converge/finish, it still generated large number of rules ($\sim 4000$). The investigation of the efficiency of these methods on purely propositional data sets (where the models employ only constants) remains an interesting direction for future research. It is clear that from our model construction and our experiments that 
\emph{relational problems are more amenable to compression} due to the inherent parameterization.

\section{Conclusion}

We have presented {\sc SCoTE} and {\sc ECoTE}, two compression algorithms based on subsumption and example based reduction to compress relational ensemble models. While {\sc SCoTE} preserves logical equivalence, {\sc ECoTE} preserves the example coverage and thus both methods do not suffer a loss in performance due to compression. We presented the algorithms from a classification perspective but it should be noted that this can be extended to perform regression or even algebraic operations of trees such as product, addition or max operations. This will allow the learnable boosted/bagged models to be used for relational reinforcement learning. Solving relational Bellman operator (REBEL)~\cite{kersting04,joshi09,wang08,sanner09} can allow for building relational RL algorithms efficiently. Understanding the relationships of the relational compression algorithms to classical propositional compression ones is an interesting future research direction.

\backmatter

\bmhead{Supplementary information}

Please see Appendix~\ref{app:perform}.

\bmhead{Acknowledgments}
Siwen Yan and Sriraam Natarajan acknowledge DARPA Minerva award FA9550-19-1-0391. Sriraam Natarajan also acknowledges the support of ARO award W911NF2010224. Prasad Tadepalli acknowledges support of DARPA contract N66001-17-2-4030 and NSF \& USDA-NIFA under grant 2021-67021-35344. Roni Khardon acknowledges support of NSF under grant IIS-2002393.

\noindent

\bibliography{sn-bibliography}

\begin{appendices}

\section{Performance} \label{app:perform}

Figures \ref{fig:Cora-sameauthor-boost} through \ref{fig:ICML-coauthor-boost} of the appendix present the train and test set performances of the original boosted trees and the learned decision lists. It can be observed that in all the domains the performance between the original trees and the compressed decision lists are not significantly different. In addition, a majority of domains clearly show the advantage of boosting where the performance increases as the number of trees increase, thus justifying the necessity of learning a powerful model and then deriving a compact representation of these models. 
\begin{figure*}[h]
  \centering
  \begin{subfigure}[t]{0.49\columnwidth}
    \centering
    \includegraphics[width=1.\linewidth]{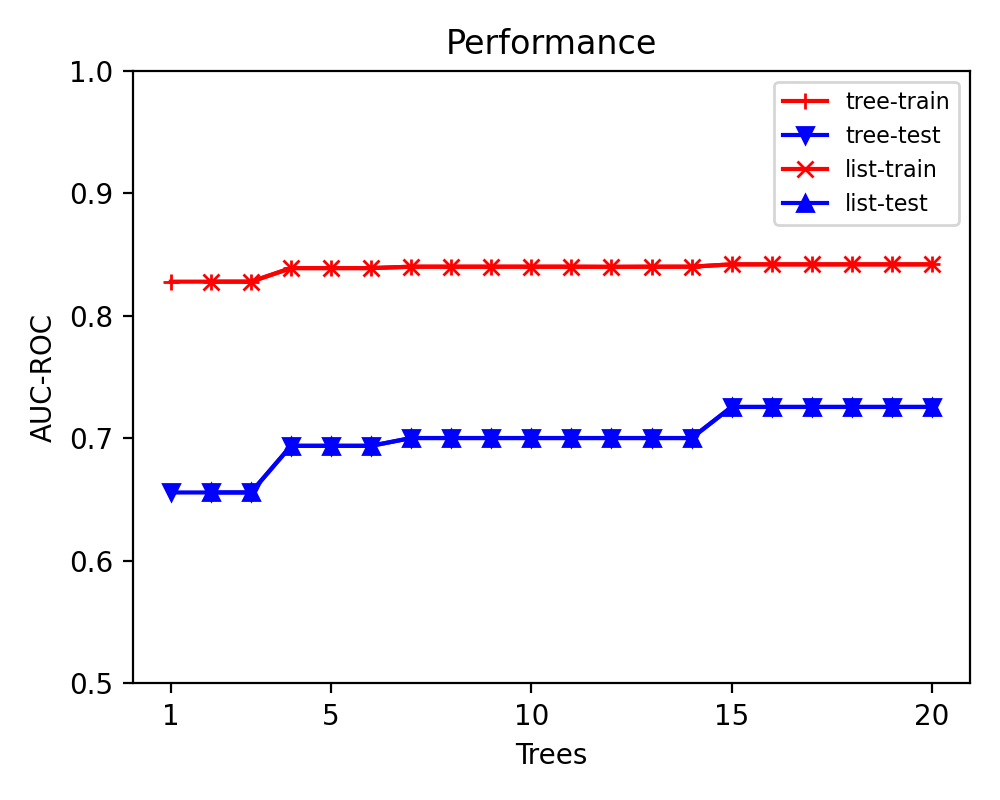}
    \label{fig:Cora-sameauthor-boost-roc}
  \end{subfigure}
  \hfill
  \begin{subfigure}[t]{0.49\columnwidth}
    \centering
    \includegraphics[width=1.\linewidth]{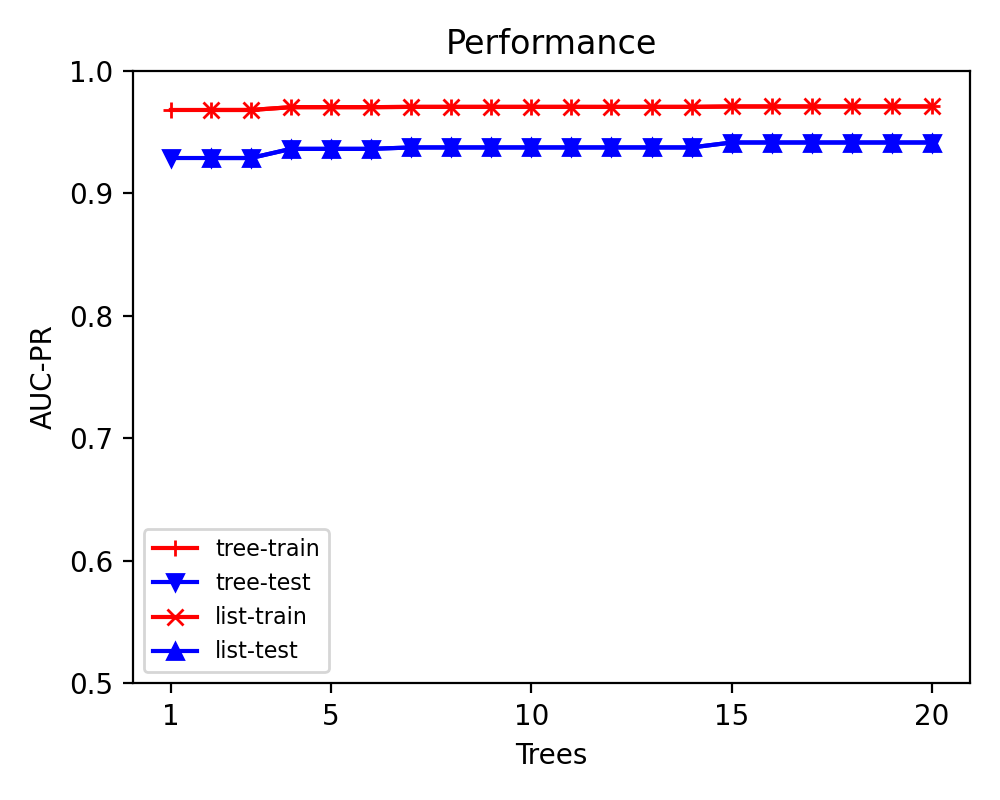}
    \label{fig:Cora-sameauthor-boost-pr}
  \end{subfigure}
  \caption{Cora SameAuthor with boosting}
 \label{fig:Cora-sameauthor-boost}
\end{figure*}

\begin{figure*}[h]
  \centering
  \begin{subfigure}[t]{0.49\columnwidth}
    \centering
    \includegraphics[width=1.\columnwidth]{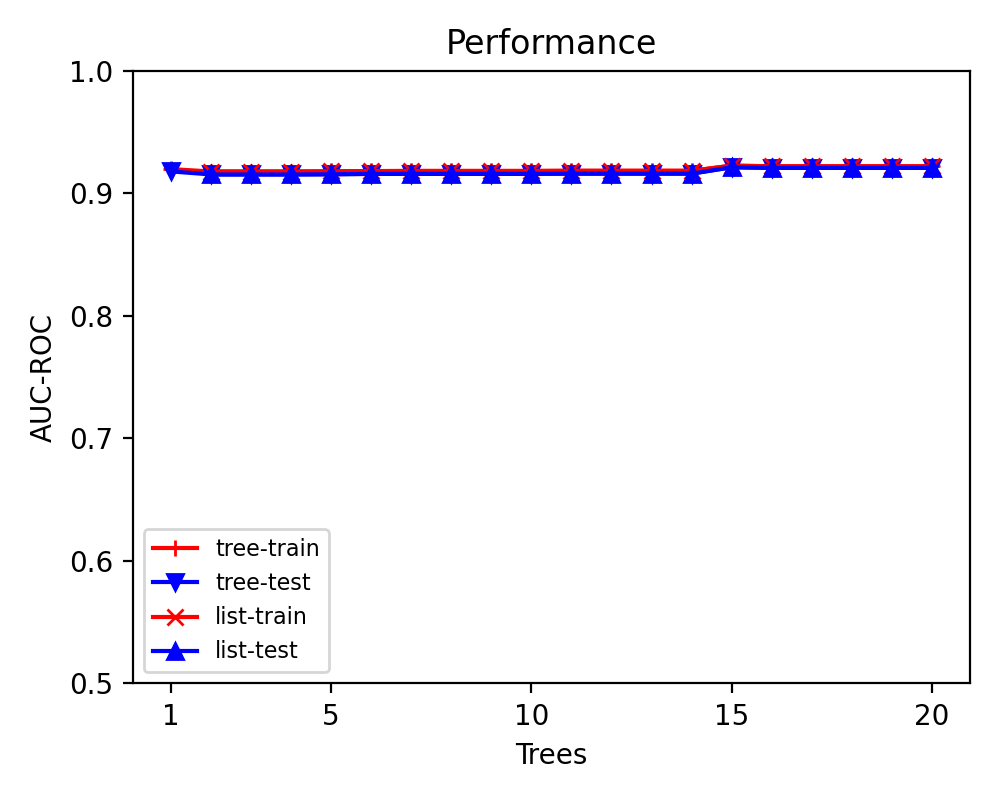}
    \label{fig:Cora-samebib-boost-roc}
  \end{subfigure}
  \hfill
  \begin{subfigure}[t]{0.49\columnwidth}
    \centering
    \includegraphics[width=1.\columnwidth]{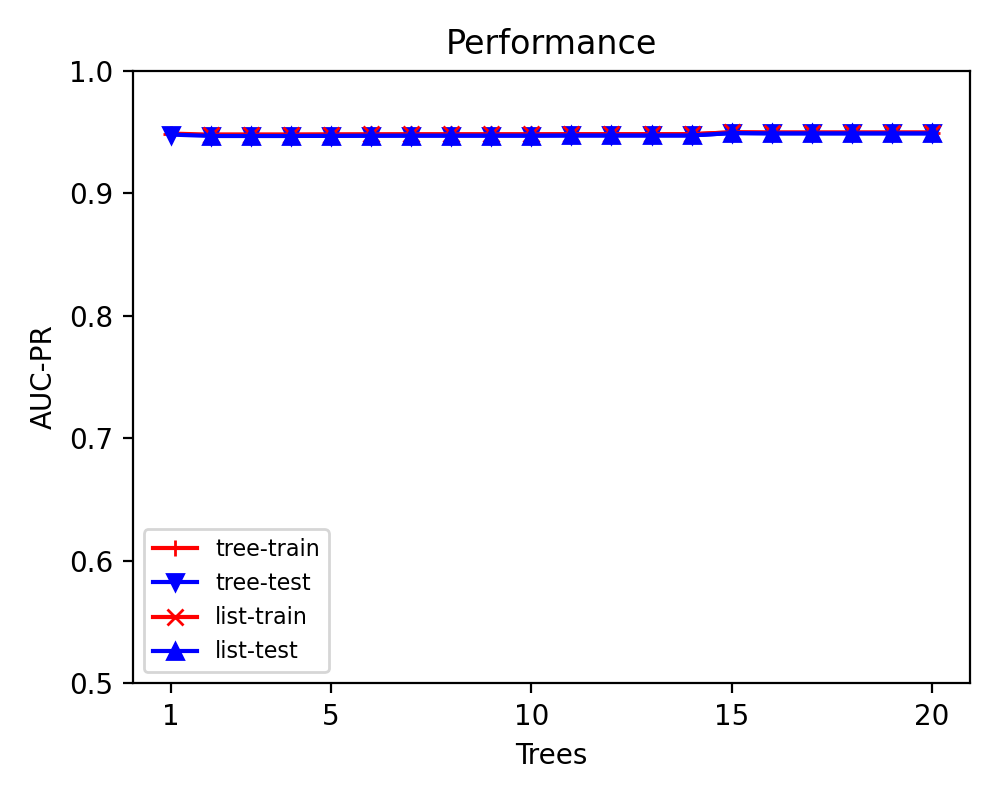}
    \label{fig:Cora-samebib-boost-pr}
  \end{subfigure}
  \caption{Cora SameBib with boosting}
 \label{fig:Cora-samebib-boost}
\end{figure*}

\begin{figure*}[h]
  \centering
  \begin{subfigure}[t]{0.49\columnwidth}
    \centering
    \includegraphics[width=1.\columnwidth]{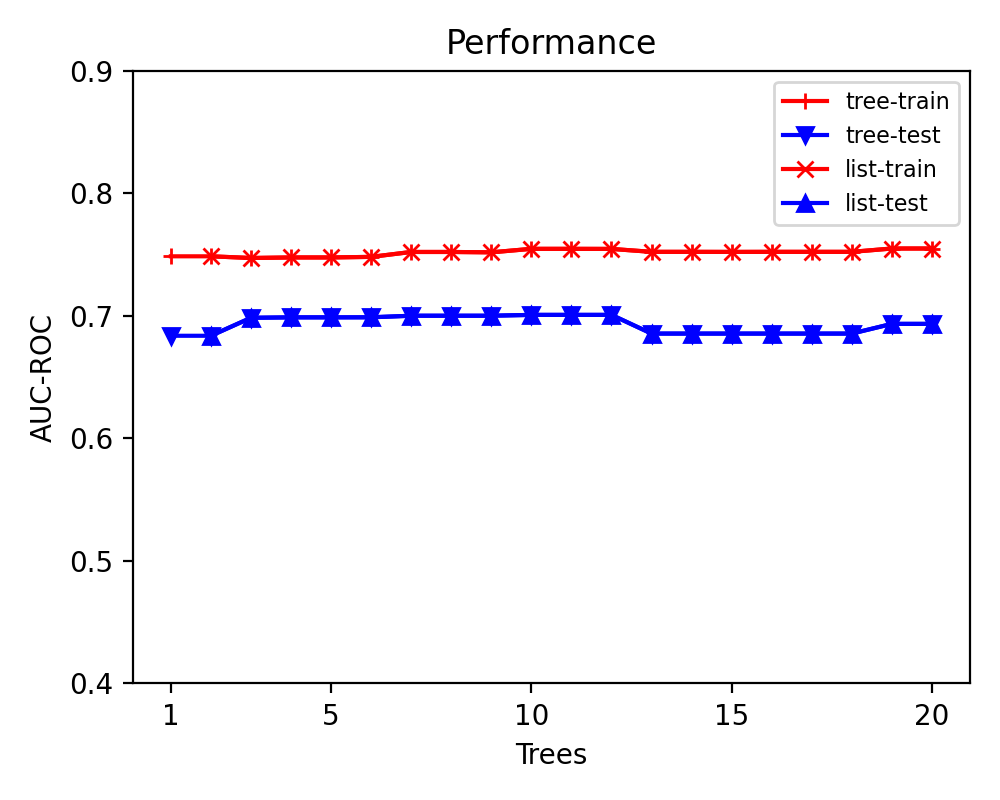}
    \label{fig:Cora-sametitle-boost-roc}
  \end{subfigure}
  \hfill
  \begin{subfigure}[t]{0.49\columnwidth}
    \centering
    \includegraphics[width=1.\columnwidth]{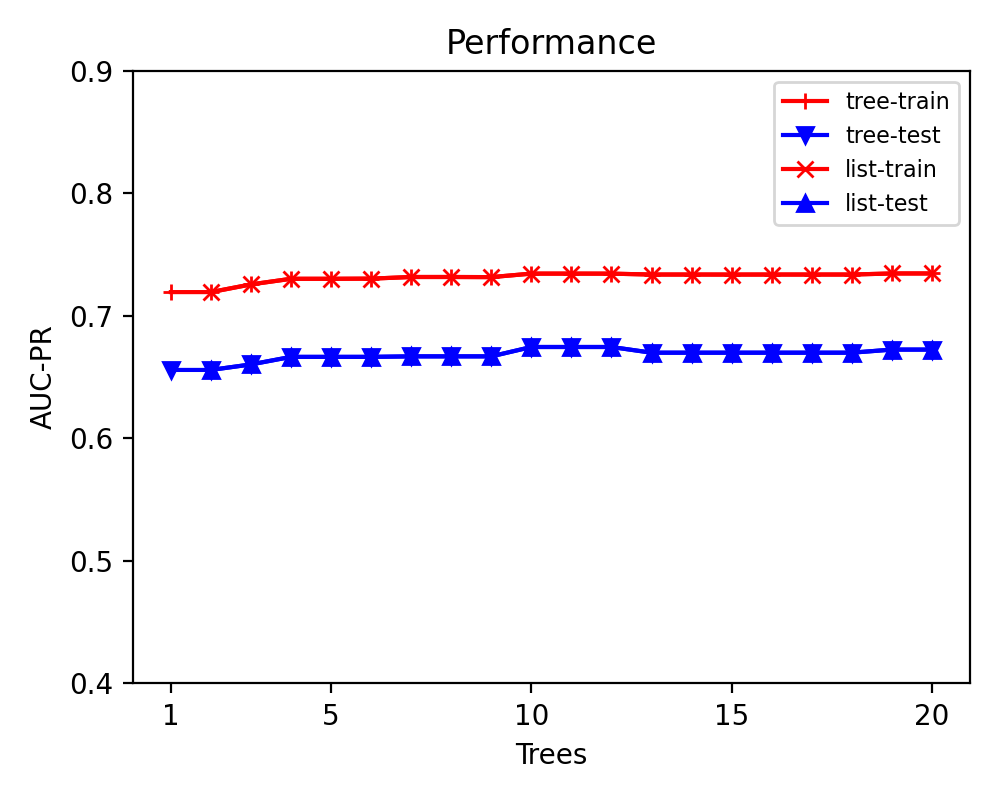}
    \label{fig:Cora-sametitle-boost-pr}
  \end{subfigure}
  \caption{Cora SameTitle with boosting}
 \label{fig:Cora-sametitle-boost}
\end{figure*}

\begin{figure*}[h]
  \centering
  \begin{subfigure}[t]{0.49\columnwidth}
    \centering
    \includegraphics[width=1.\columnwidth]{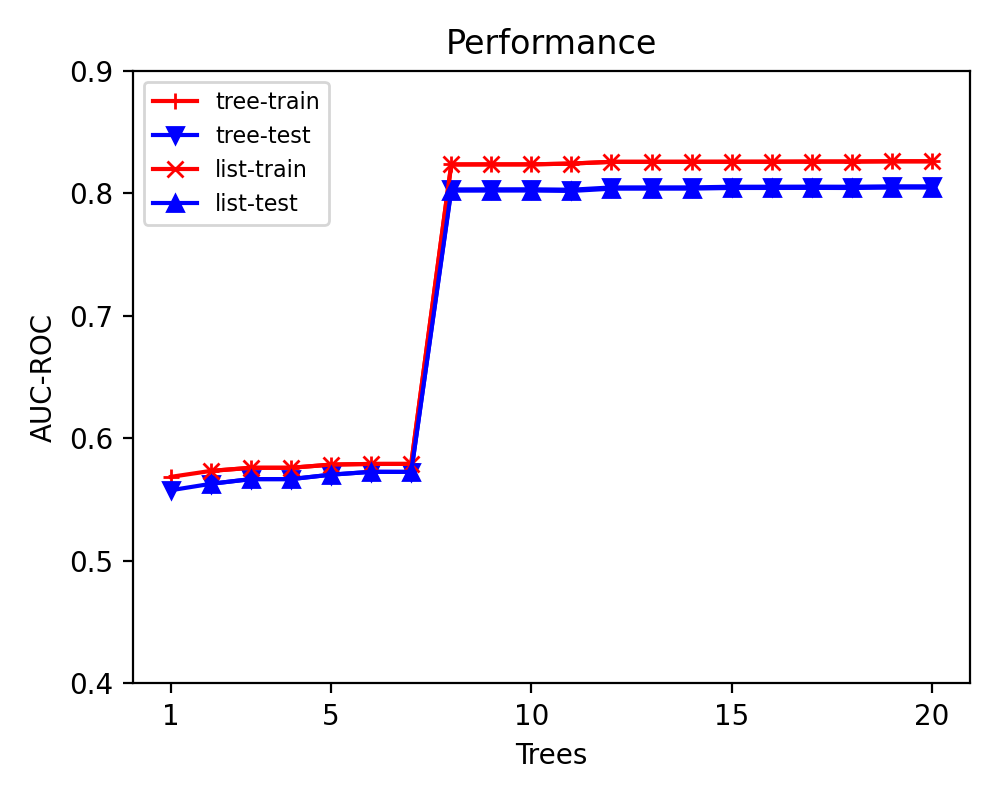}
    \label{fig:Cora-samevenue-boost-roc}
  \end{subfigure}
  \hfill
  \begin{subfigure}[t]{0.49\columnwidth}
    \centering
    \includegraphics[width=1.\columnwidth]{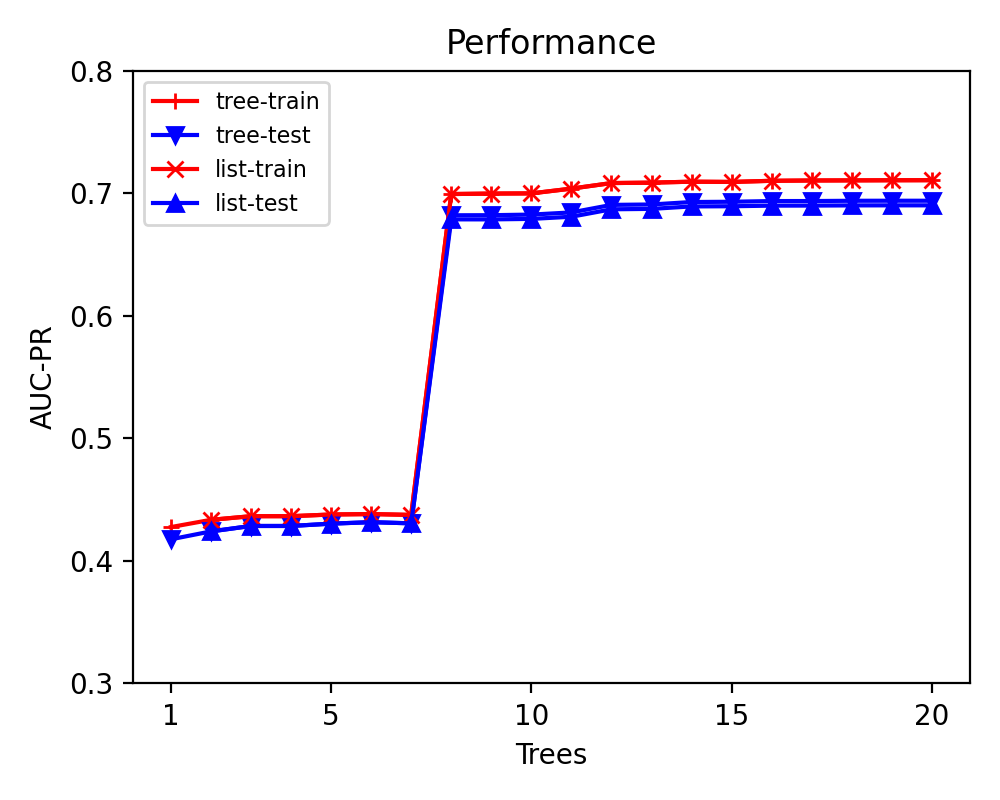}
    \label{fig:Cora-samevenue-boost-pr}
  \end{subfigure}
  \caption{Cora SameVenue with boosting}
 \label{fig:Cora-samevenue-boost}
\end{figure*}

\begin{figure*}[h]
  \centering
  \begin{subfigure}[t]{0.49\columnwidth}
    \centering
    \includegraphics[width=1.\columnwidth]{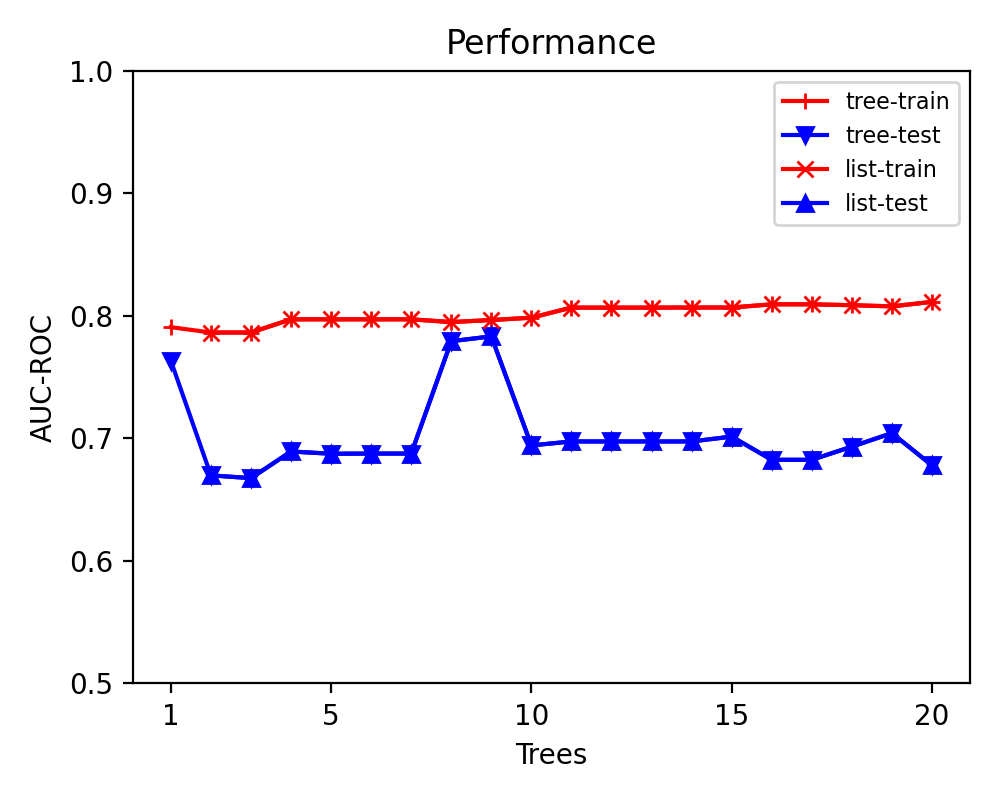}
    \label{fig:WebKB-faculty-boost-roc}
  \end{subfigure}
  \hfill
  \begin{subfigure}[t]{0.49\columnwidth}
    \centering
    \includegraphics[width=1.\columnwidth]{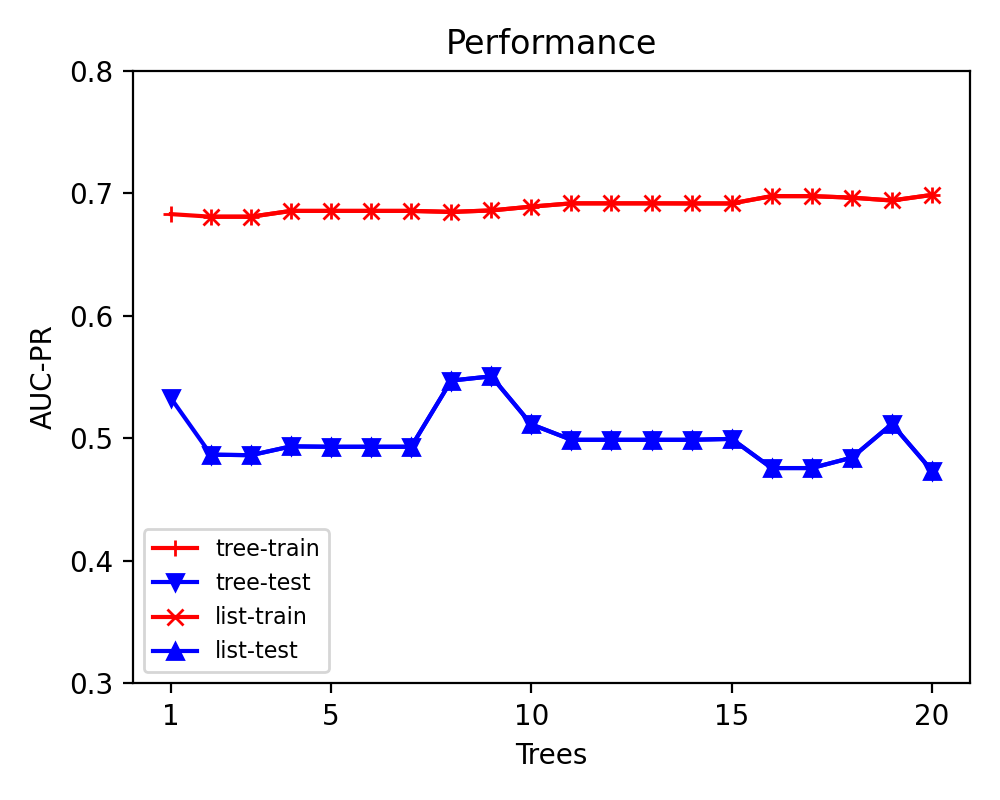}
    \label{fig:WebKB-faculty-boost-pr}
  \end{subfigure}
  \caption{WebKB Faculty with boosting}
 \label{fig:WebKB-faculty-boost}
\end{figure*}

\begin{figure*}[h]
  \centering
  \begin{subfigure}[t]{0.49\columnwidth}
    \centering
    \includegraphics[width=1.\columnwidth]{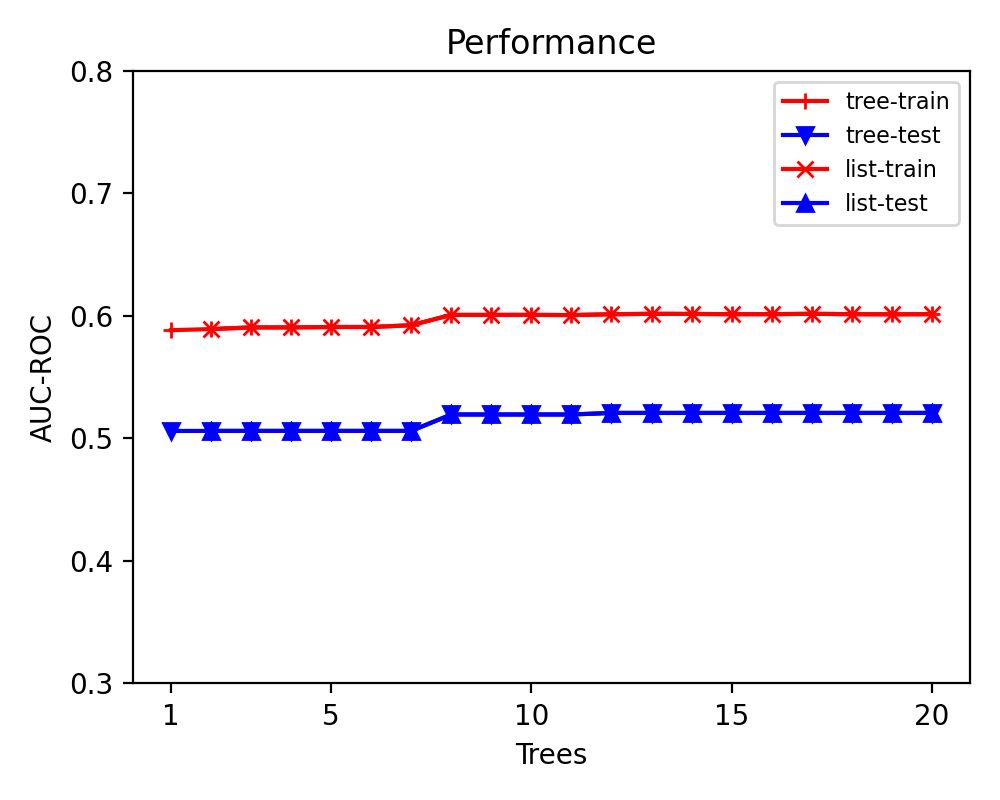}
    \label{fig:WebKB-courseprof-boost-roc}
  \end{subfigure}
  \hfill
  \begin{subfigure}[t]{0.49\columnwidth}
    \centering
    \includegraphics[width=1.\columnwidth]{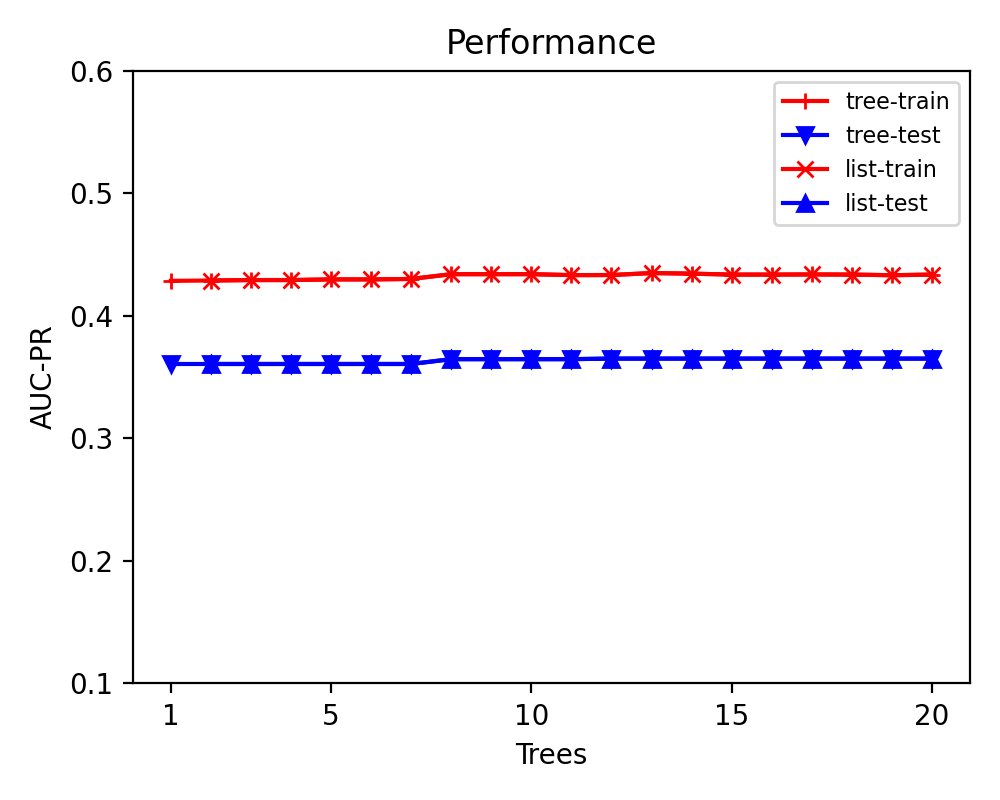}
    \label{fig:WebKB-courseprof-boost-pr}
  \end{subfigure}
  \caption{WebKB CourseProf with boosting}
 \label{fig:WebKB-courseprof-boost}
\end{figure*}

\begin{figure*}[h]
  \centering
  \begin{subfigure}[t]{0.49\columnwidth}
    \centering
    \includegraphics[width=1.\columnwidth]{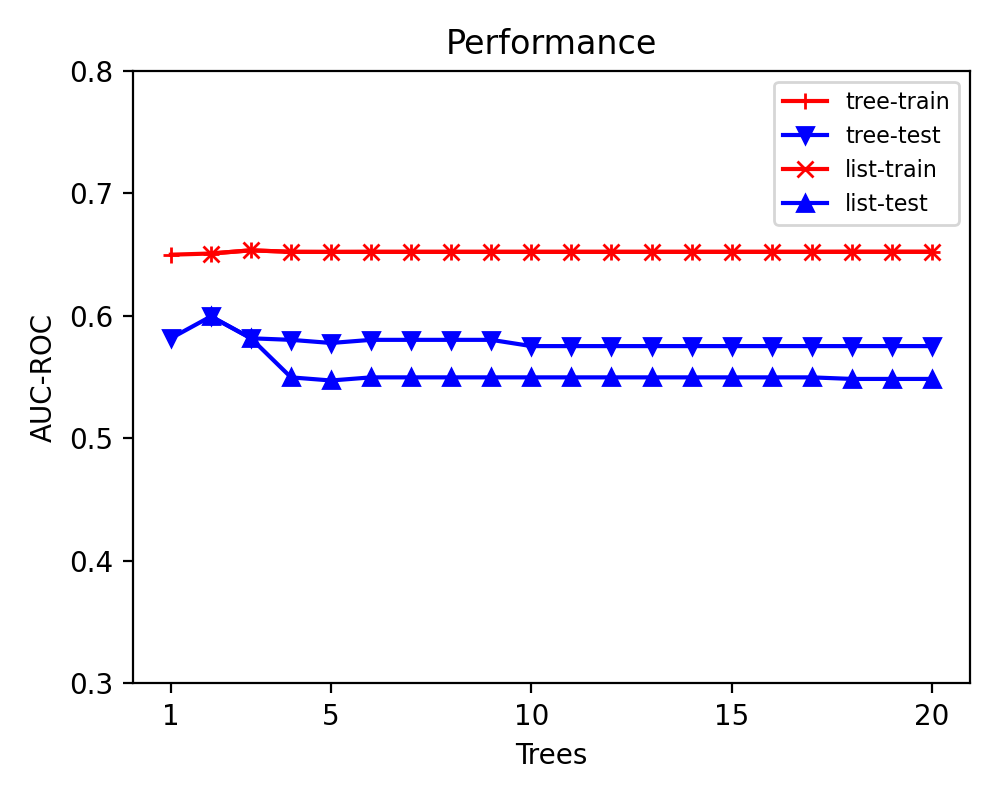}
    \label{fig:WebKB-courseta-boost-roc}
  \end{subfigure}
  \hfill
  \begin{subfigure}[t]{0.49\columnwidth}
    \centering
    \includegraphics[width=1.\columnwidth]{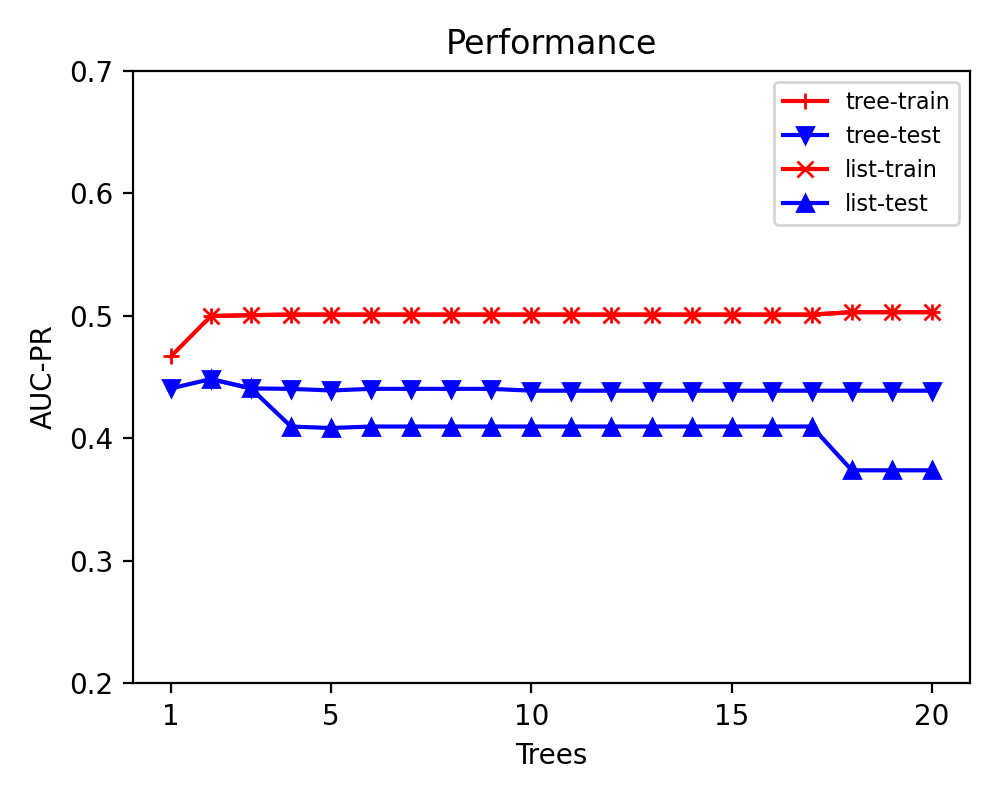}
    \label{fig:WebKB-courseta-boost-pr}
  \end{subfigure}
  \caption{WebKB CourseTA with boosting}
 \label{fig:WebKB-courseta-boost}
\end{figure*}

\begin{figure*}[h]
  \centering
  \begin{subfigure}[t]{0.49\columnwidth}
    \centering
    \includegraphics[width=1.\columnwidth]{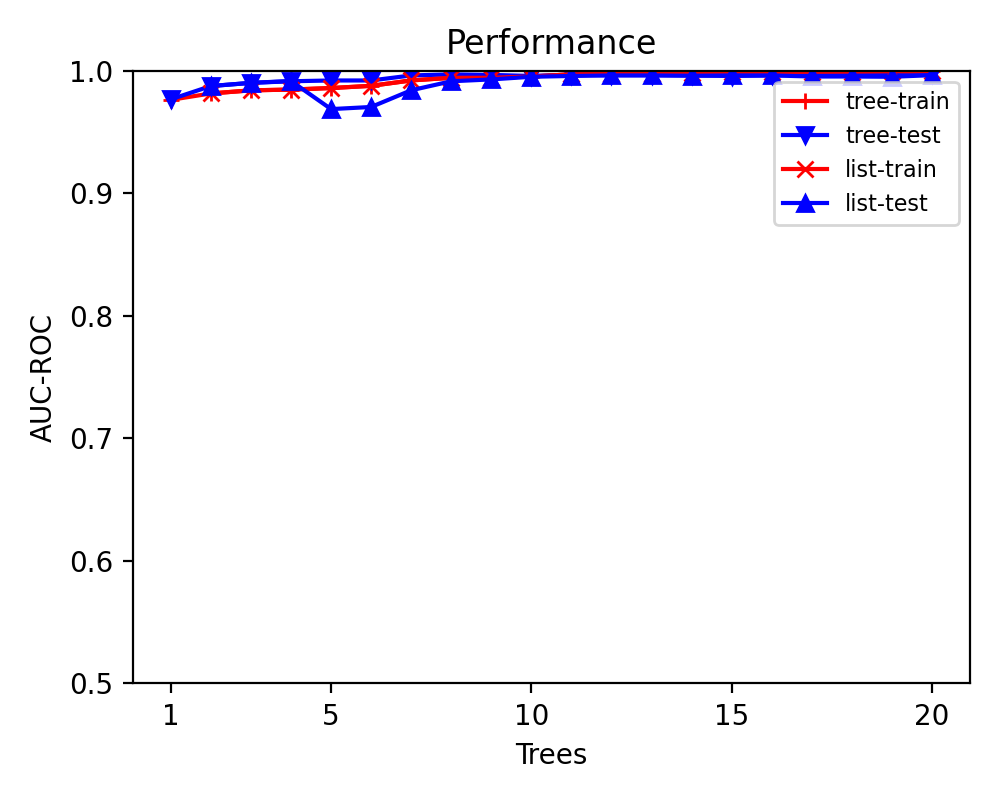}
    \label{fig:UW-CSE-advisedby-boost-roc}
  \end{subfigure}
  \hfill
  \begin{subfigure}[t]{0.49\columnwidth}
    \centering
    \includegraphics[width=1.\columnwidth]{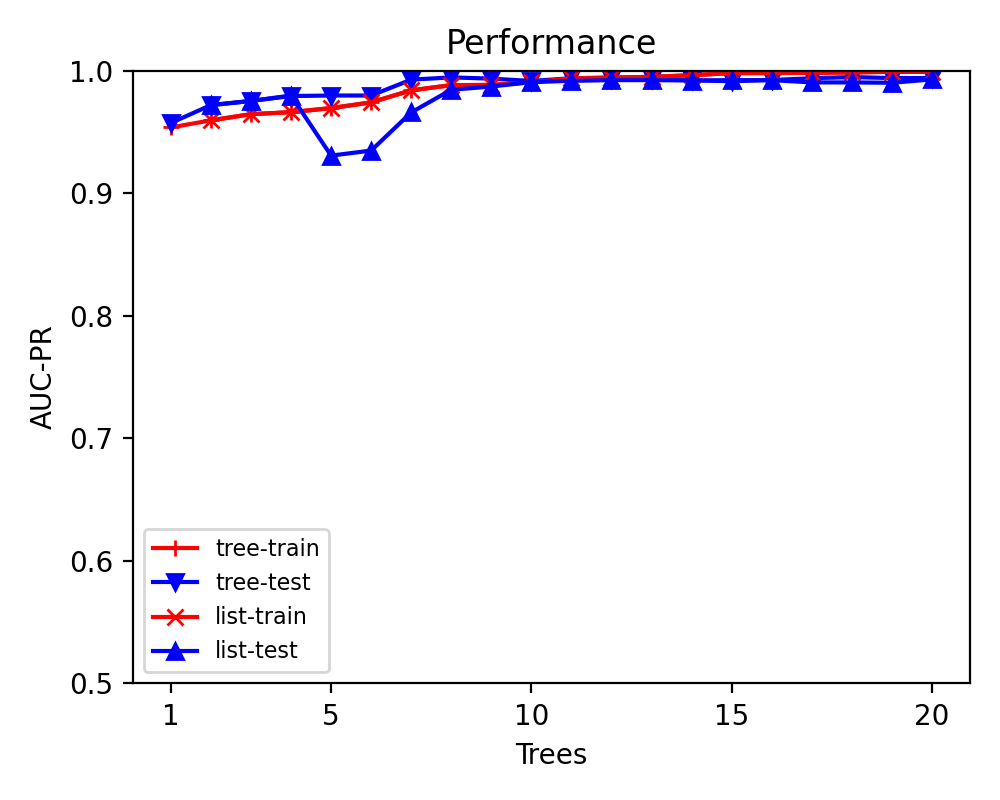}
    \label{fig:UW-CSE-advisedby-boost-pr}
  \end{subfigure}
  \caption{UW-CSE AdvisedBy with boosting}
 \label{fig:UW-CSE-advisedby-boost}
\end{figure*}

\begin{figure*}[h]
  \centering
  \begin{subfigure}[t]{0.49\columnwidth}
    \centering
    \includegraphics[width=1.\columnwidth]{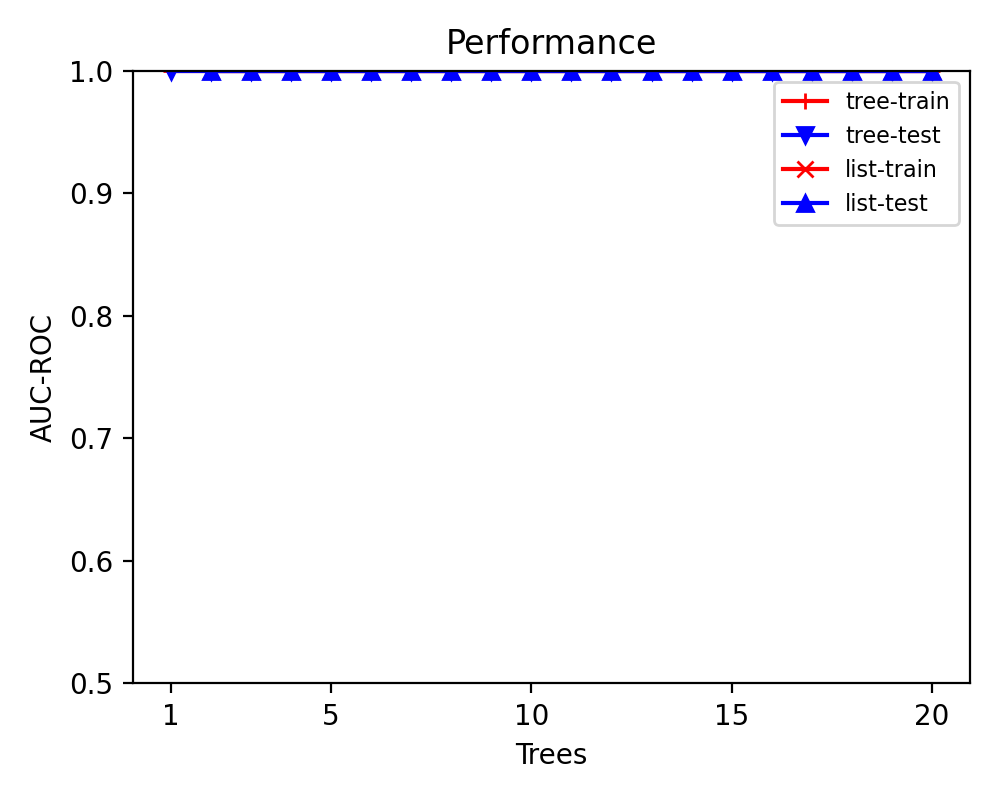}
    \label{fig:IMDB-workedunder-boost-roc}
  \end{subfigure}
  \hfill
  \begin{subfigure}[t]{0.49\columnwidth}
    \centering
    \includegraphics[width=1.\columnwidth]{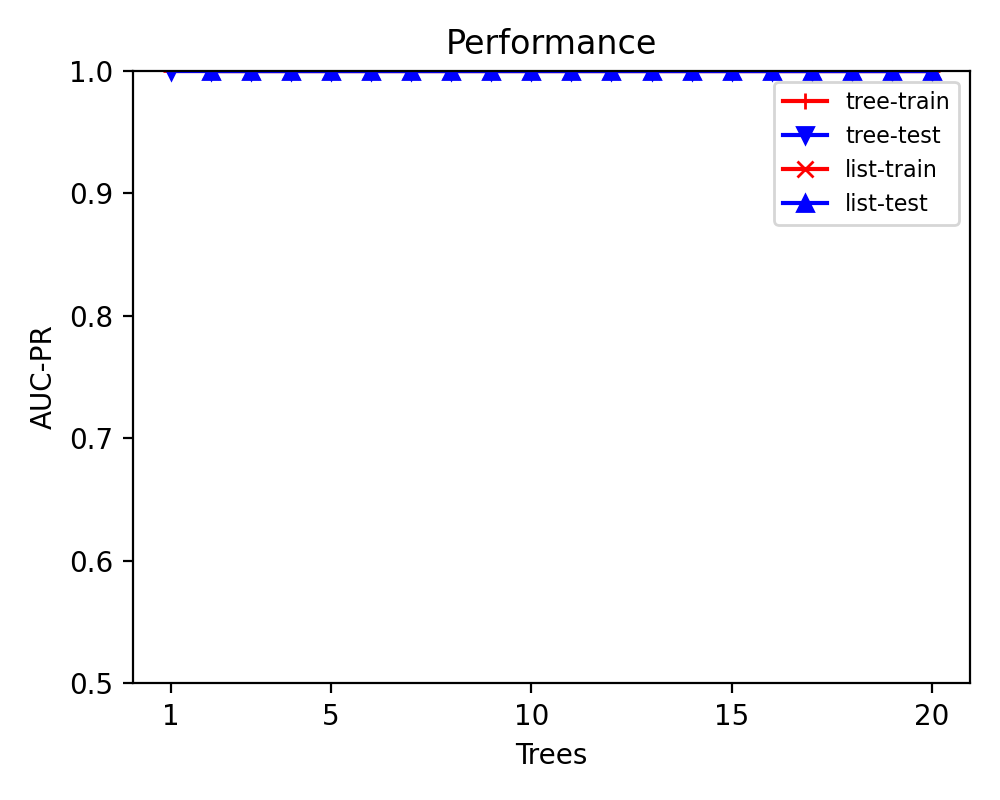}
    \label{fig:IMDB-workedunder-boost-pr}
  \end{subfigure}
  \caption{IMDB WorkedUnder with boosting}
 \label{fig:IMDB-workedunder-boost}
\end{figure*}

\begin{figure*}[h]
  \centering
  \begin{subfigure}[t]{0.49\columnwidth}
    \centering
    \includegraphics[width=1.\columnwidth]{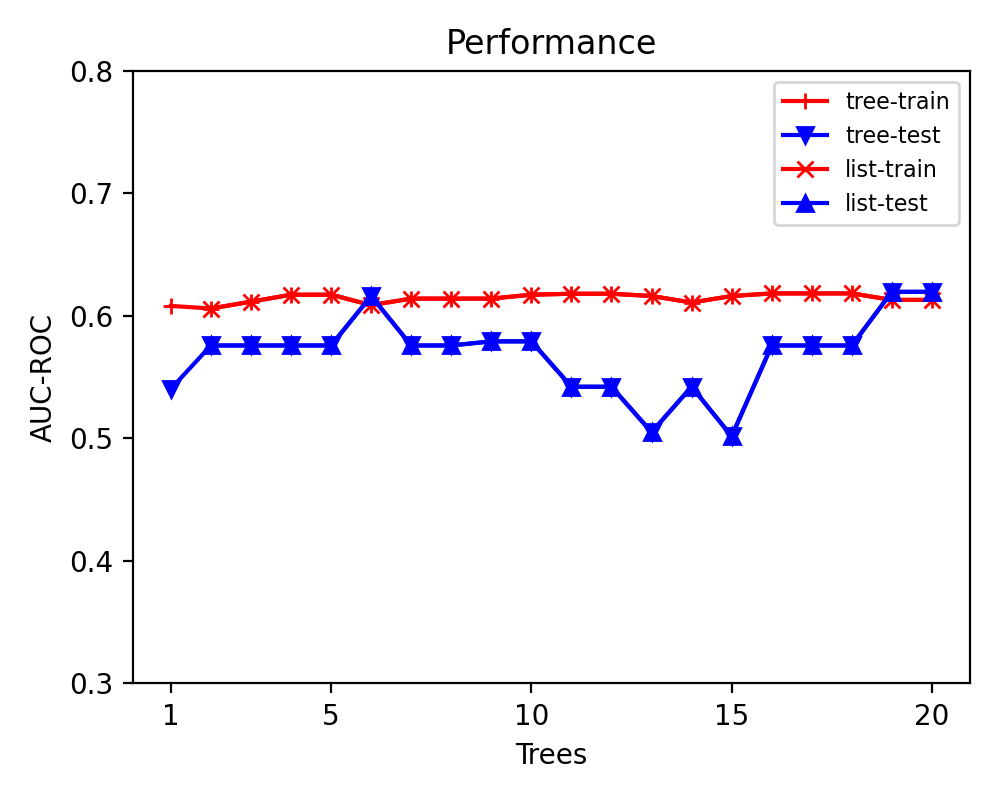}
    \label{fig:IMDB-genre-boost-roc}
  \end{subfigure}
  \hfill
  \begin{subfigure}[t]{0.49\columnwidth}
    \centering
    \includegraphics[width=1.\columnwidth]{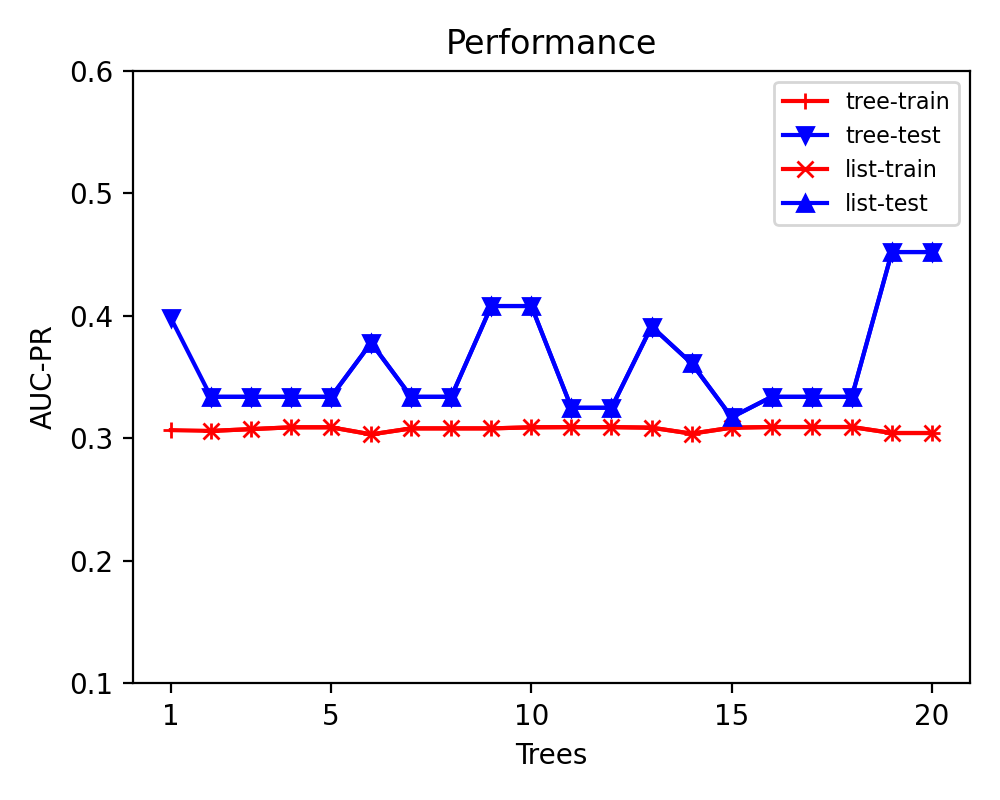}
    \label{fig:IMDB-genre-boost-pr}
  \end{subfigure}
  \caption{IMDB Genre with boosting}
 \label{fig:IMDB-genre-boost}
\end{figure*}

\begin{figure*}[h]
  \centering
  \begin{subfigure}[t]{0.49\columnwidth}
    \centering
    \includegraphics[width=1.\columnwidth]{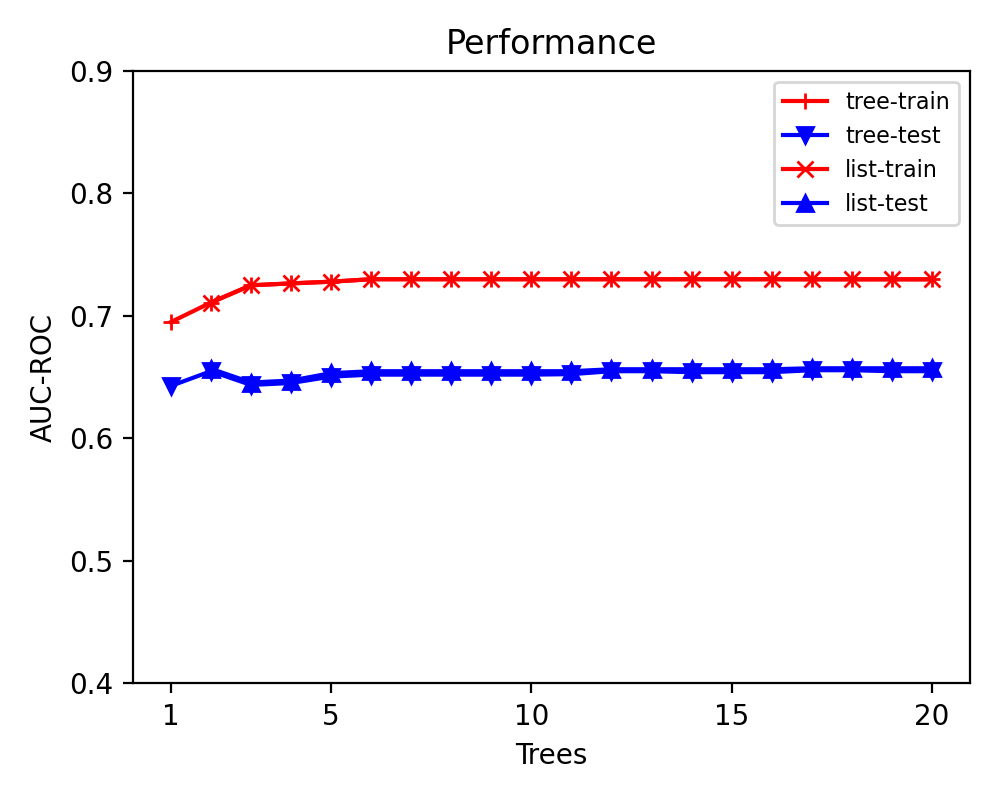}
    \label{fig:IMDB-female_gender-boost-roc}
  \end{subfigure}
  \hfill
  \begin{subfigure}[t]{0.49\columnwidth}
    \centering
    \includegraphics[width=1.\columnwidth]{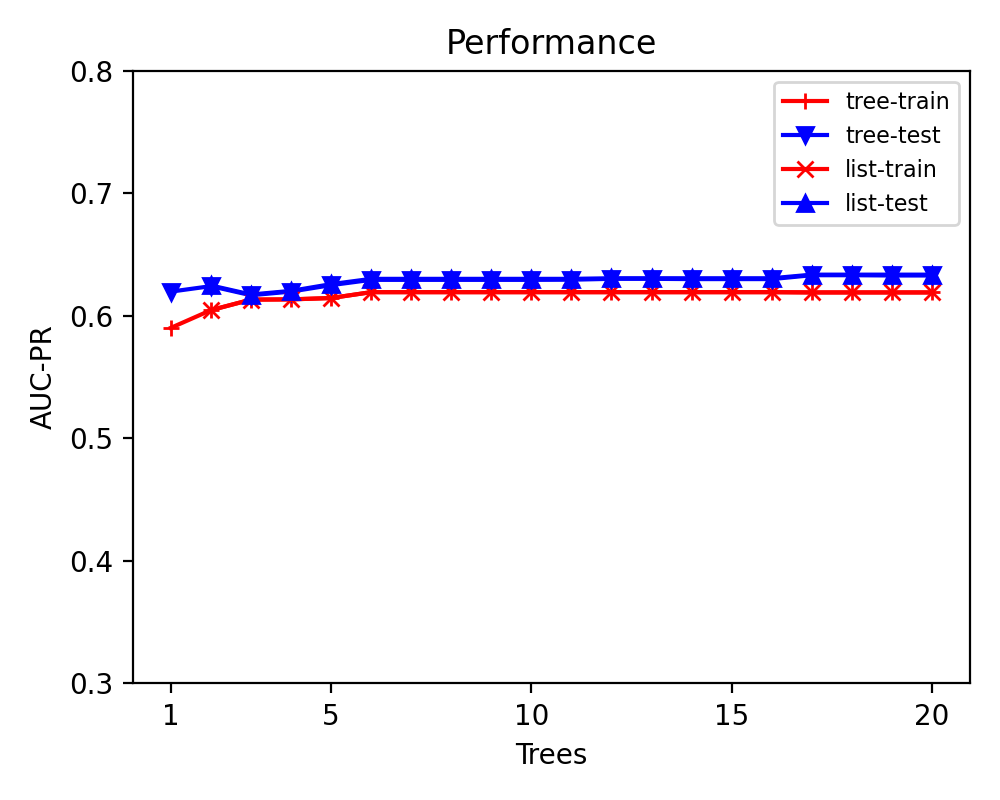}
    \label{fig:IMDB-female_gender-boost-pr}
  \end{subfigure}
  \caption{IMDB FemaleGender with boosting}
 \label{fig:IMDB-female_gender-boost}
\end{figure*}

\begin{figure*}[h]
  \centering
  \begin{subfigure}[t]{0.49\columnwidth}
    \centering
    \includegraphics[width=1.\columnwidth]{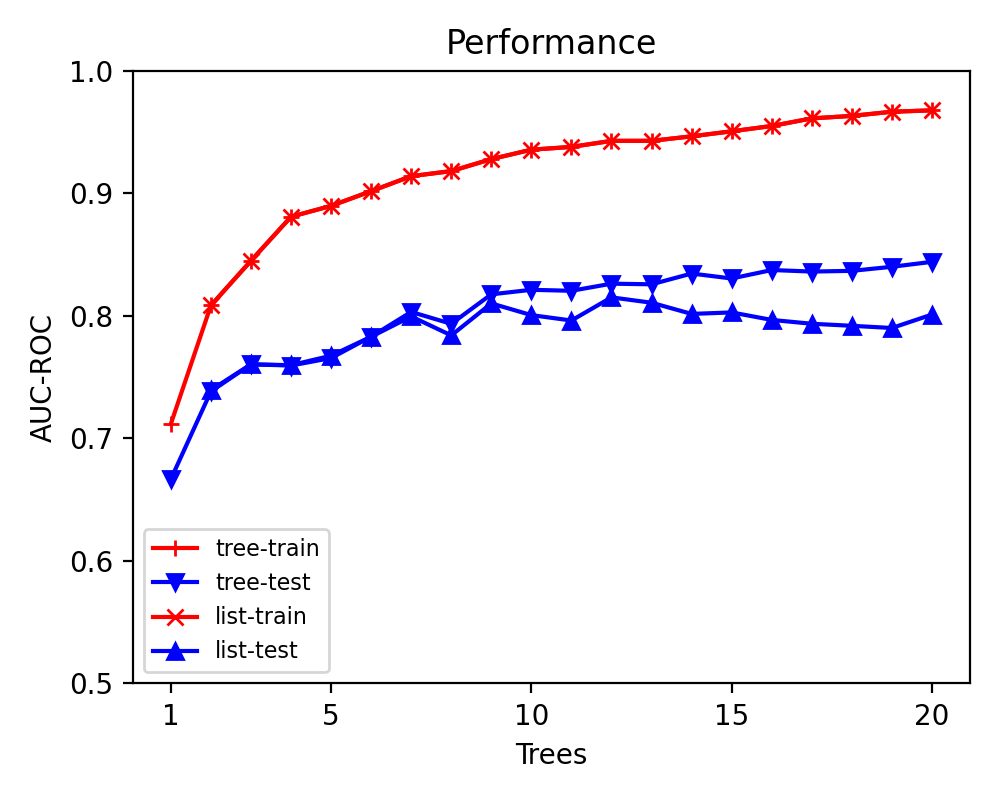}
    \label{fig:ICML-coauthor-boost-roc}
  \end{subfigure}
  \hfill
  \begin{subfigure}[t]{0.49\columnwidth}
    \centering
    \includegraphics[width=1.\columnwidth]{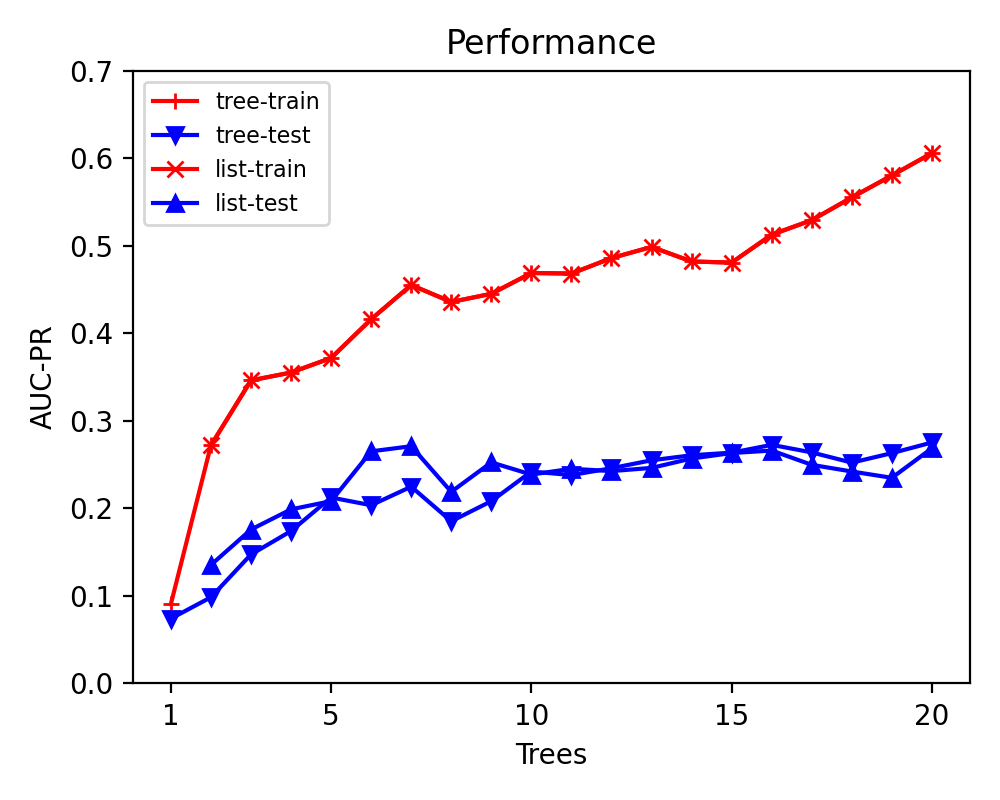}
    \label{fig:ICML-coauthor-boost-pr}
  \end{subfigure}
  \caption{ICML CoAuthor with boosting}
 \label{fig:ICML-coauthor-boost}
\end{figure*}

\end{appendices}

\end{document}